\documentclass[twocolumn, switch]{article}
\usepackage[T1]{fontenc}
\usepackage[utf8]{inputenc}
\usepackage{color}
\usepackage{array}
\usepackage{wrapfig}
\usepackage{units}
\usepackage{mathtools}
\usepackage{bm}
\usepackage{amsmath}
\usepackage{amsthm}
\usepackage{amssymb}
\usepackage{graphicx}
\usepackage{microtype}
\usepackage[unicode=true,
 bookmarks=false,
 breaklinks=false,pdfborder={0 0 1},backref=section,colorlinks=false]
 {hyperref}

\makeatletter

\providecommand{\tabularnewline}{\\}

\usepackage{paralist}
\theoremstyle{plain}
\newtheorem{prop}{\protect\propositionname}
\theoremstyle{remark}
\newtheorem{rem}{\protect\remarkname}
\theoremstyle{definition}
\newtheorem{defn}{\protect\definitionname}
\theoremstyle{plain}
\newtheorem{thm}{\protect\theoremname}
\theoremstyle{definition}
\newtheorem{problem}{\protect\problemname}

\usepackage{preprint}

\usepackage{amsmath,amsfonts,bm}

\def\eqref#1{equation~\ref{#1}}

\def\1{\bm{1}}

\def\rb{{\textnormal{b}}}

\def\rf{{\textnormal{f}}}

\def\rt{{\textnormal{t}}}

\DeclareMathAlphabet{\mathsfit}{\encodingdefault}{\sfdefault}{m}{sl}
\SetMathAlphabet{\mathsfit}{bold}{\encodingdefault}{\sfdefault}{bx}{n}

\newcommand{\E}{\mathbb{E}}

\newcommand{\R}{\mathbb{R}}

\newcommand{\softmax}{\mathrm{softmax}}

\DeclareMathOperator*{\argmax}{arg\,max}
\DeclareMathOperator*{\argmin}{arg\,min}

\DeclareMathOperator{\sign}{sign}

\usepackage{hyperref}
\usepackage{url}

\title{Experimental Design for Overparameterized Learning with Application to Single Shot Deep Active Learning}

\author{%
  Neta Shoham 
\\	
   School of Mathematical Sciences\\
   Tel Aviv University \\
   Tel Aviv, Israel \\
  \texttt{shohamne@gmail.com} \\
   \And
   Haim Avron \\
   School of Mathematical Sciences\\
   Tel Aviv University \\
   Tel Aviv, Israel \\
   \texttt{haimav@tauex.tau.ac.il} \\
}

\author{Neta Shoham, Haim Avron
\\
Tel Aviv University \\
Tel Aviv, Israel \\
\texttt{shohamne,haimav@tauex.tau.ac.il} \\
}

\mathtoolsset{showonlyrefs}
\usepackage{wrapfig}

\makeatother

\providecommand{\definitionname}{Definition}
\providecommand{\problemname}{Problem}
\providecommand{\propositionname}{Proposition}
\providecommand{\remarkname}{Remark}
\providecommand{\theoremname}{Theorem}

\begin{document}
\global\long\def\R{\mathbb{R}}%

\global\long\def\C{\mathbb{C}}%

\global\long\def\N{\mathbb{N}}%

\global\long\def\e{{\mathbf{e}}}%

\global\long\def\et#1{{\e(#1)}}%

\global\long\def\ef{{\mathbf{\et{\cdot}}}}%

\global\long\def\x{{\mathbf{x}}}%

\global\long\def\xt#1{{\x(#1)}}%

\global\long\def\xf{{\mathbf{\xt{\cdot}}}}%

\global\long\def\d{{\mathbf{d}}}%

\global\long\def\w{{\mathbf{w}}}%

\global\long\def\b{{\mathbf{b}}}%

\global\long\def\u{{\mathbf{u}}}%

\global\long\def\y{{\mathbf{y}}}%

\global\long\def\n{{\mathbf{n}}}%

\global\long\def\k{{\mathbf{k}}}%

\global\long\def\yt#1{{\y(#1)}}%

\global\long\def\yf{{\mathbf{\yt{\cdot}}}}%

\global\long\def\z{{\mathbf{z}}}%

\global\long\def\v{{\mathbf{v}}}%

\global\long\def\h{{\mathbf{h}}}%

\global\long\def\s{{\mathbf{s}}}%

\global\long\def\c{{\mathbf{c}}}%

\global\long\def\p{{\mathbf{p}}}%

\global\long\def\f{{\mathbf{f}}}%

\global\long\def\rb{{\mathbf{r}}}%

\global\long\def\rt#1{{\rb(#1)}}%

\global\long\def\rf{{\mathbf{\rt{\cdot}}}}%

\global\long\def\mat#1{{\ensuremath{\bm{\mathrm{#1}}}}}%

\global\long\def\vec#1{{\ensuremath{\bm{\mathrm{#1}}}}}%

\global\long\def\matN{\ensuremath{{\bm{\mathrm{N}}}}}%

\global\long\def\matX{\ensuremath{{\bm{\mathrm{X}}}}}%

\global\long\def\X{\ensuremath{{\bm{\mathrm{X}}}}}%

\global\long\def\matK{\ensuremath{{\bm{\mathrm{K}}}}}%

\global\long\def\K{\ensuremath{{\bm{\mathrm{K}}}}}%

\global\long\def\matA{\ensuremath{{\bm{\mathrm{A}}}}}%

\global\long\def\A{\ensuremath{{\bm{\mathrm{A}}}}}%

\global\long\def\matB{\ensuremath{{\bm{\mathrm{B}}}}}%

\global\long\def\B{\ensuremath{{\bm{\mathrm{B}}}}}%

\global\long\def\matC{\ensuremath{{\bm{\mathrm{C}}}}}%

\global\long\def\C{\ensuremath{{\bm{\mathrm{C}}}}}%

\global\long\def\matD{\ensuremath{{\bm{\mathrm{D}}}}}%

\global\long\def\D{\ensuremath{{\bm{\mathrm{D}}}}}%

\global\long\def\matE{\ensuremath{{\bm{\mathrm{E}}}}}%

\global\long\def\E{\ensuremath{{\bm{\mathrm{E}}}}}%

\global\long\def\matF{\ensuremath{{\bm{\mathrm{F}}}}}%

\global\long\def\F{\ensuremath{{\bm{\mathrm{F}}}}}%

\global\long\def\matH{\ensuremath{{\bm{\mathrm{H}}}}}%

\global\long\def\H{\ensuremath{{\bm{\mathrm{H}}}}}%

\global\long\def\matP{\ensuremath{{\bm{\mathrm{P}}}}}%

\global\long\def\P{\ensuremath{{\bm{\mathrm{P}}}}}%

\global\long\def\matU{\ensuremath{{\bm{\mathrm{U}}}}}%

\global\long\def\matV{\ensuremath{{\bm{\mathrm{V}}}}}%

\global\long\def\V{\ensuremath{{\bm{\mathrm{V}}}}}%

\global\long\def\matW{\ensuremath{{\bm{\mathrm{W}}}}}%

\global\long\def\matM{\ensuremath{{\bm{\mathrm{M}}}}}%

\global\long\def\M{\ensuremath{{\bm{\mathrm{M}}}}}%

\global\long\def\calH{{\cal H}}%

\global\long\def\calY{{\cal Y}}%

\global\long\def\calP{{\cal P}}%

\global\long\def\calX{{\cal X}}%

\global\long\def\calS{{\cal S}}%

\global\long\def\calT{{\cal T}}%

\global\long\def\Normal{{\cal \mathcal{N}}}%

\global\long\def\matQ{{\mat Q}}%

\global\long\def\Q{{\mat Q}}%

\global\long\def\matR{\mat R}%

\global\long\def\matS{\mat S}%

\global\long\def\matY{\mat Y}%

\global\long\def\matI{\mat I}%

\global\long\def\I{\mat I}%

\global\long\def\matJ{\mat J}%

\global\long\def\matZ{\mat Z}%

\global\long\def\Z{\mat Z}%

\global\long\def\matW{{\mat W}}%

\global\long\def\W{{\mat W}}%

\global\long\def\matL{\mat L}%

\global\long\def\S#1{{\mathbb{S}_{N}[#1]}}%

\global\long\def\IS#1{{\mathbb{S}_{N}^{-1!}[#1]}}%

\global\long\def\PN{\mathbb{P}_{N}}%

\global\long\def\TNormS#1{\|#1\|_{2}^{2}}%

\global\long\def\ITNormS#1{\|#1\|_{2}^{-2}}%

\global\long\def\ONorm#1{\|#1\Vert_{1}}%

\global\long\def\TNorm#1{\|#1\|_{2}}%

\global\long\def\InfNorm#1{\|#1\|_{\infty}}%

\global\long\def\FNorm#1{\|#1\|_{F}}%

\global\long\def\FNormS#1{\|#1\|_{F}^{2}}%

\global\long\def\UNorm#1{\|#1\|_{\matU}}%

\global\long\def\UNormS#1{\|#1\|_{\matU}^{2}}%

\global\long\def\UINormS#1{\|#1\|_{\matU^{-1}}^{2}}%

\global\long\def\ANorm#1{\|#1\|_{\matA}}%

\global\long\def\BNorm#1{\|#1\|_{\mat B}}%

\global\long\def\ANormS#1{\|#1\|_{\matA}^{2}}%

\global\long\def\AINormS#1{\|#1\|_{\matA^{-1}}^{2}}%

\global\long\def\T{\textsc{T}}%

\global\long\def\conj{\textsc{*}}%

\global\long\def\pinv{\textsc{+}}%

\global\long\def\Expect#1{{\mathbb{E}}\left[#1\right]}%

\global\long\def\ExpectC#1#2{{\mathbb{E}}_{#1}\left[#2\right]}%

\global\long\def\VarC#1#2{{\mathbb{\mathrm{Var}}}_{#1}\left[#2\right]}%

\global\long\def\dotprod#1#2#3{(#1,#2)_{#3}}%

\global\long\def\dotprodN#1#2{(#1,#2)_{{\cal N}}}%

\global\long\def\dotprodH#1#2{\langle#1,#2\rangle_{{\cal {\cal H}}}}%

\global\long\def\dotprodsqr#1#2#3{(#1,#2)_{#3}^{2}}%

\global\long\def\Trace#1{{\bf Tr}\left(#1\right)}%

\global\long\def\nnz#1{{\bf nnz}\left(#1\right)}%

\global\long\def\MSE#1{{\bf MSE}\left(#1\right)}%

\global\long\def\WMSE#1{{\bf WMSE}\left(#1\right)}%

\global\long\def\EWMSE#1{{\bf EWMSE}\left(#1\right)}%

\global\long\def\nicehalf{\nicefrac{1}{2}}%

\global\long\def\argmin{\operatornamewithlimits{argmin}}%

\global\long\def\argmax{\operatornamewithlimits{argmax}}%

\global\long\def\norm#1{\Vert#1\Vert}%

\global\long\def\sign{\operatorname{sign}}%

\global\long\def\diag{\operatorname{diag}}%

\global\long\def\VOPT{\operatorname\{VOPT\}}%

\global\long\def\dist{\operatorname{dist}}%

\global\long\def\diag{\operatorname{diag}}%

\global\long\def\sp{\operatorname{span}}%

\global\long\def\onehot{\operatorname{onehot}}%

\global\long\def\softmax{\operatorname{softmax}}%

\newcommand*\diff{\mathop{}\!\mathrm{d}} 

\global\long\def\dd{\diff}%

\global\long\def\whatlambda{\w_{\lambda}}%

\global\long\def\Plambda{\mat P_{\lambda}}%

\global\long\def\Pperplambda{\left(\mat I-\Plambda\right)}%

\global\long\def\Mlambda{\matM_{\lambda}}%

\global\long\def\Mlambdafull{\matM+\lambda\matI}%

\global\long\def\Mlambdafullinv{\left(\matM+\lambda\matI\right)^{-1}}%

\global\long\def\Mdaggerlambda{{\mathbf{\mat M_{\lambda}^{+}}}}%

\global\long\def\Xdaggerlambda{{\mathbf{\mat X_{\lambda}^{+}}}}%

\global\long\def\XT{{\mathbf{X}^{\T}}}%

\global\long\def\XXT{{\matX\mat X^{\T}}}%

\global\long\def\XTX{{\matX^{\T}\mat X}}%

\global\long\def\VT{{\mathbf{V}^{\T}}}%

\global\long\def\VVT{{\matV\mat V^{\T}}}%

\global\long\def\VTV{{\matV^{\T}\mat V}}%

\global\long\def\varphibar#1#2{{\bar{\varphi}_{#1,#2}}}%

\global\long\def\varphilambda{{\varphi_{\lambda}}}%

\twocolumn[ 
\begin{@twocolumnfalse}
\maketitle
\begin{abstract}
The impressive performance exhibited by modern machine learning models hinges on the ability to train such models on a very large amounts of labeled data. However, since access to large volumes of labeled data is often limited or expensive, it is desirable to alleviate this bottleneck by carefully curating the training set. Optimal experimental design is a well-established paradigm for selecting data point to be labeled so to maximally inform the learning process. Unfortunately, classical theory on optimal experimental design focuses on selecting examples in order to learn underparameterized (and thus, non-interpolative) models, while modern machine learning models such as deep neural networks are overparameterized, and oftentimes are trained to be interpolative. As such, classical experimental design methods are not applicable in many modern learning setups. Indeed, the predictive performance of underparameterized models tends to be variance dominated, so classical experimental design focuses on variance reduction, while the predictive performance of overparameterized models can also be, as is shown in this paper, bias dominated or of mixed nature. In this paper we propose a design strategy that is well suited for overparameterized regression and interpolation, and we demonstrate the applicability of our method in the context of deep learning by proposing a new algorithm for single shot deep active learning.
\end{abstract}
\vspace{0.35cm}
\end{@twocolumnfalse}
]

\section{Introduction}

The impressive performance exhibited by modern machine learning models
hinges on the ability to train the aforementioned models on very large
amounts of labeled data. In practice, in many real world scenarios,
even when raw data exists aplenty, acquiring labels might prove challenging
and/or expensive. This severely limits the ability to deploy machine
learning capabilities in real world applications. This bottleneck
has been recognized early on, and methods to alleviate it have been
suggested.  Most relevant for our work is the large body of research
on\emph{ active learning} or\emph{ optimal experimental design}, which
aims at selecting data point to be labeled so to maximally inform
the learning process. Disappointedly, active learning techniques seem
to deliver mostly lukewarm benefits in the context of deep learning.

One possible reason why experimental design has so far failed to make
an impact in the context of deep learning is that such models are
\emph{overparameterized}, and oftentimes are trained to be\emph{ interpolative}~\cite{zhang2016understanding},
i.e., they are trained so that a perfect fit of the training data
is found. This raises a conundrum: the classical perspective on statistical
learning theory is that overfitting should be avoided since there
is a tradeoff between the fit and complexity of the model. This conundrum
is exemplified by the \emph{double descent phenomena}~\cite{belkin2019two,bartlett2019benign},
namely when fixing the model size and increasing the amount of training
data, the predictive performance initially goes down, and then starts
to go up, exploding when the amount of training data approaches the
model complexity, and then starts to descend again. This runs counter
to statistical intuition which says that more data implies better
learning. Indeed, when using interpolative models, more data can hurt~\cite{nakkiran2019deep}!
This phenomena is exemplified in the curve labeled ``Random Selection''
in Figure~\ref{fig:opening}. Figure~\ref{fig:opening} explores
the predictive performance of various designs when learning a linear
regression model and varying the amount of training data with responses.

\begin{figure}
\centering{}%
\begin{minipage}[t]{0.95\columnwidth}%
\begin{center}
\includegraphics[width=1\textwidth]{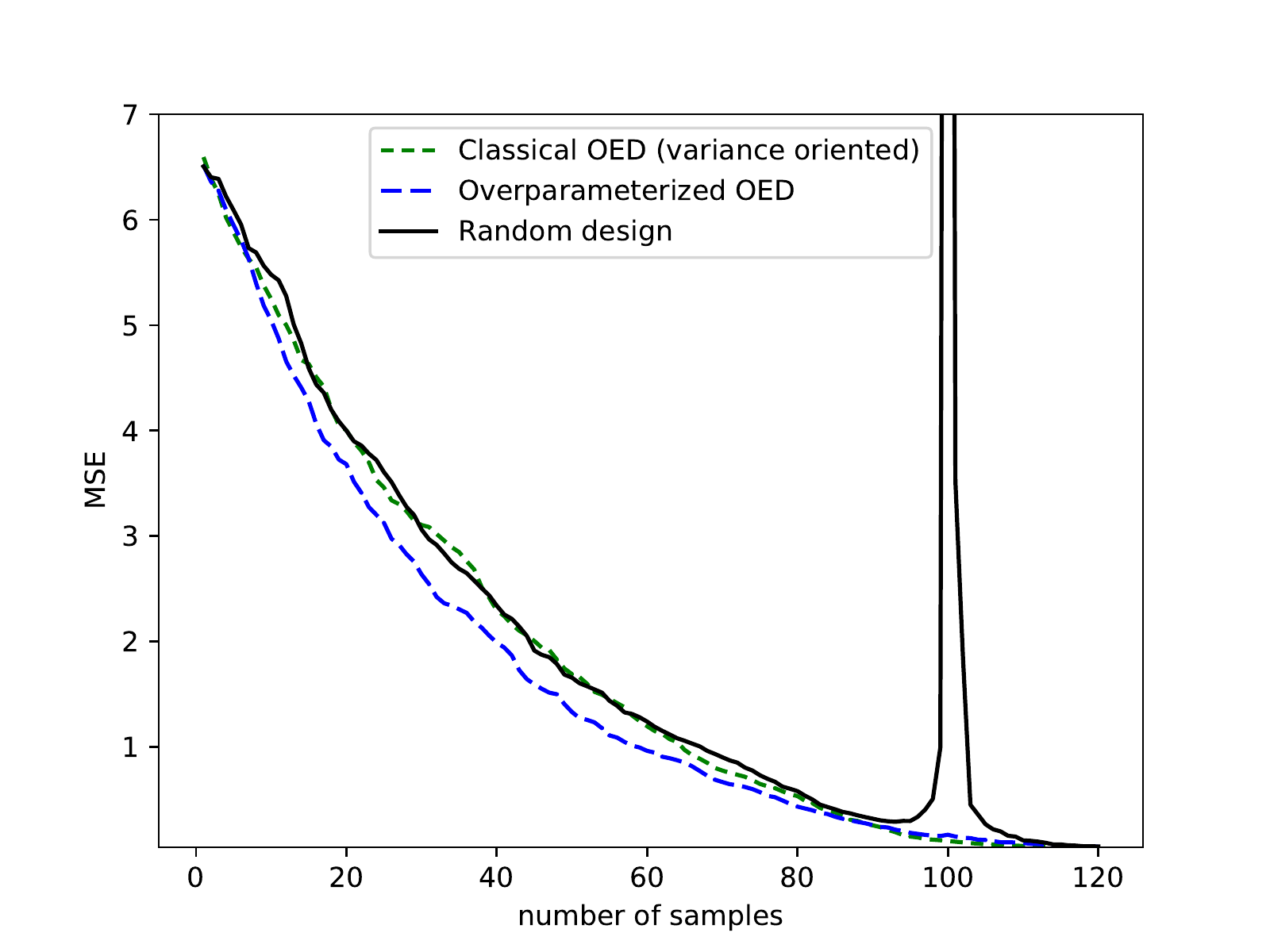}
\par\end{center}%
\end{minipage}\caption{\label{fig:opening}MSE of a minimum norm linear interpolative model.
We use synthetic data of dimension 100. The full description is in
Appendix \ref{sec:Experiment-setup-of}.}
\end{figure}

The fact that more data can hurt further motivates experimental design
in the interpolative regime. Presumably, if data is carefully curated,
more data should never hurt. Unfortunately, classical optimal experimental
design focuses on the underparameterized (and thus, non-interpolative)
case. As such, the theory reported in the literature is often not
applicable in the interpolative regime. As our analysis shows (see
Section~\ref{sec:overparmeterized-oed}), the prediction error of
interpolative models can either be\emph{ bias dominated} (the first
descent phase, i.e., when training size is very small compared to
the number of parameters),\emph{ variance dominated} (near equality
between number of training examples and number of parameters) or of
\emph{mixed nature}. However, properly trained underparameterized
models tend to have prediction error which is variance dominated,
so classical experimental design focuses on variance reduction. As
such, naively using classical optimality criteria, such as V-optimality
(the one most relevant for generalization error) or others, in the
context of interpolation, tends to produce poor results when prediction
error is bias dominated or of mixed nature. This is exemplified in
the curve labeled ``Classical OED'' in Figure~\ref{fig:opening}.

The goal of this paper is to understand these regimes, and to propose
an experimental design strategy that is well suited for overparameterized
models. Like many recent work that attempt to understand the double
descent phenomena by analyzing underdetermined linear regression,
we too use a simple linear regression model in our analysis of experimental
design in the overparameterized case (however, we also consider kernel
ridge regression, not only linear interpolative models). We believe
that understanding experimental design in the overparameterized linear
regression case is a prelude to designing effective design algorithms
for deep learning. Indeed, recent theoretical results showed a deep
connection between deep learning and kernel learning via the so-called
Neural Tangent Kernel \cite{jacot2018neural,arora2019exact,lee2019wide}.
Based on this connection, and as a proof-of-concept, we propose a
new algorithm for single shot deep active learning. 

Let us now summarize our contributions:
\begin{compactitem}
\item We analyze the prediction error of learning overparameterized linear
models for a given fixed design, revealing three possible regimes
that call for different design criteria: bias dominated, variance
dominated, and mixed nature. We also reveal an interesting connection
between overparameterized experimental design and the\emph{ column
subset selection problem}~\cite{boutsidis2008improved},\emph{ transductive
experimental design}~\cite{yu2006active}, and\emph{ coresets}~\cite{sener2017active}.
We also extend our approach to kernel ridge regression.
\item We propose a novel greedy algorithm for finding designs for overparameterized
linear models. As exemplified in the curve labeled ``Overparameterized
OED'', our algorithm is sometimes able to mitigate the double descent
phenomena, while still performing better than classical OED (though
no formal proof of this fact is provided). 
\item We show how our algorithm can also be applied for kernel ridge regression,
and report experiments which show that when the number of parameters
is in a sense infinite, our algorithm is able to find designs that
are better than state of the art. We propose a new algorithm for
single shot deep active learning, a scaracly treated problem so far,
and demonstrate its effectiveness on MNIST.
\end{compactitem}

\subsection{Related Work.}

The phenomena of benign overfitting and double descent was firstly
recognized in DNNs \cite{zhang2016understanding}, and later discussed
and analyzed in the context of linear models \cite{zhang2016understanding,belkin2018understand,belkin2018reconciling,belkin2019two,bartlett2019benign}.
Recently there is also a growing interest in the related phenomena
of ``more data can hurt'' \cite{nakkiran2019deep,nakkiran2019more,nakkiran2020optimal,loog2019minimizers}.
Other works discussed the need to consider zero or negative regularization
coefficient for large real life linear models \cite{kobak2019optimal}.

Experimental design is an well established paradigm in statistics,
extensively covered in the literature for the linear case \cite{pukelsheim2006optimal}
and the non linear case \cite{pronzato2013design}. Its application
to pool based active learning with batch acquisitions was explored
in \cite{yu2006active} for linear models and in \cite{hoi2006batch}
for logistic regression. It was also proposed in the context of deep
learning~\cite{sourati2018active}. Another related line of work
is recent work by Haber and Horesh on experimental design for ill-posed
inverse problems\cite{haber2008numerical,haber2012numerical,horesh2010optimal}.
Active learning in the context of overparameterized learning was explored
in \cite{karzand2020maximin}, however their approach differs from
ours significantly since it is based on artificially completing the
labels using a minimax approach.

In the context of \emph{Laplacian regularized Least Squares} (LapRLS),
which is a generalization of ridge regression, in \cite{gu2012selective}
it is shown rigorously that the criterion suggested in \cite{yu2006active}
is justified as a bound for both the bias and variance components
of the expected error. We farther show that this bound is in some
sense tight only if the parameter norm is one and the noise variance
equals the $l_{2}$ penalty coefficient. In addition we postulate
and show experimentally that in the overparameterized case using a
bias dominant criterion is preferable. Another case in which the bias
term does not vanish is when the model is misspecified. For linear
\cite{sugiyama2005active} and generalized linear \cite{bach2007active}
models this case has been tackled with reweighing of the loss function.

A popular modern approach for pool based active learning with batch
acquisition is coresets \cite{sener2017active,geifman2017deep,ash2019deep,pinsler2019bayesian}.
This approach has been used in the context of active learning for
DNNs.

\section{\label{sec:Preliminaries}Background: Underparameterized V-Optimal
Experimental Design}

First we recall the classical experimental design criterion of V-optimality
which applies in the underparameterized setting.

Consider a noisy linear response model $\text{\ensuremath{y=\x^{\T}\w+\epsilon}}$,
where $\epsilon\sim\Normal(0,\sigma^{2})$ and $\w\in\R^{d}$ and
assume we are given with some data points $\x_{1},\dots,\x_{n}$,
for which we obtained independent responses, $y_{i}=\x_{i}^{\T}\w+\epsilon_{i}$.
Consider the underparameterized case, i.e. $n\geq d$, and furthermore
assume that the set $\{\x_{1},\dots,\x_{n}\}$ contains at least $d$
independent vectors. The best linear unbiased estimator $\hat{\w}$
of $\w$ according to the Gauss-Markov theorem is given by: $\hat{\w}=\arg\min_{\w}\TNormS{\matX\w-\y}=\matX^{+}\y$
where $\matX\in\R^{n\times d}$ is a matrix whose rows are $\x_{1},\dots,\x_{n}$,
$\y=\left[y_{1}\dots y_{n}\right]^{\T}\in\R^{n}$ and $\matX^{\pinv}$
is the Moore-Pensrose pseudoinverse of $\matX$. It is well known
that in this case $\hat{\w}-\w$ is a normal random vector with zero
mean and covariance matrix $\sigma^{2}\mat M^{-1}$, where $\matM=\XTX$
is the Fisher information matrix. This implies that $\hat{y}(\x)-y(\x)$
is also a normal variable with zero mean and variance equal to $\sigma^{2}\x^{\T}\matM^{-1}\x$.

Assume also that data points ($\x$'s) are sampled independently from
some distribution $\rho.$ Now, we can further define the excess risk:
\[
R(\hat{\w})=\ExpectC{\x\sim\rho}{(\x^{\T}\w-\x^{\T}\hat{\w})^{2}}
\]
and calculate its expectation: 
\begin{align}
\ExpectC{\epsilon}{R(\hat{\w})} & =\ExpectC{\x\sim\rho}{\VarC{\epsilon}{y(\x)-\hat{y}(\x)}}\nonumber \\
 & =\ExpectC{\x\sim\rho}{\sigma^{2}\x^{\T}\matM^{-1}\x}\nonumber \\
 & =\Trace{\sigma^{2}\matM^{-1}\matC_{\rho}}\label{eq:risk_under}
\end{align}
where $\mat C_{\rho}$ is the uncentered second moment matrix of $\rho$:
\[
\matC_{\rho}\coloneqq\ExpectC{\x\sim\rho}{\x\x^{\T}}.
\]

Eq.~\eqref{eq:risk_under} motivates the so-called \emph{V-optimal
design} criterion: select the dataset $\x_{1},\dots,\x_{n}$ so that
$\varphi(\matM)\coloneqq\Trace{\matM^{-1}\matC_{\rho}}$ is minimized
(if we do not have access to $\matC_{\rho}$ then it is possible to
estimate it by drawing samples from $\rho$). In doing so, we are
trying to minimize the expected (with respect to the noise $\epsilon$)
average (with respect to the data $\x$) prediction variance, since
the risk is composed solely from it (due to the fact that the estimator
is unbiased). As we shall see, this is in contrast with the overparameterized
case, in which the estimator is biased.

V-optimality is only one instance of various statistical criteria
used in experimental design. In general experimental design, the focus
is on minimizing a preselected criteria $\varphi\left(\matM\right)$
\cite{pukelsheim2006optimal}. For example in D-optimal design, $\varphi(\matM)=\det(\matM^{-1})$
and in A-optimal design $\varphi(\matM)=\Trace{\matM^{-1}}$. However,
since minimizing the V-optimality criterion corresponds to minimizing
the risk, it is more appropriate when assessing the predictive performance
of machine learning models.

\section{\label{sec:overparmeterized-oed}Overparameterized Experimental Design
Criteria}

In this section we derive an expression for the risk in the overparameterized
case, i.e. like Eq.~\eqref{eq:risk_under} but also for the case
that $n\leq d$ (our expressions also hold for $n>d$). This, in turn,
leads to an experimental design criteria analogous to V-optimality,
but relevant for overparamterized modeling as well. We design a novel
algorithm based on this criteria in subsequent sections.

\subsection{Overparameterized Regression and Interpolation}

When $n\geq d$ there is a natural candidate for $\hat{\w}$: the
best unbiased linear estimator $\matX^{\pinv}\y$\footnote{In practice, when $n$ is only mildly bigger than $d$ it is usually
better to regularize the problem.}. However, when $d>n$ there is no longer a unique minimizer of $\TNormS{\matX\w-\y}$
as there is an infinite amount of interpolating $\w$'s, i.e. $\w$'s
such that $\matX\w=\y$ (the last statement makes the mild additional
assumption that $\matX$ has full row rank). One natural strategy
for dealing with the non-uniqueness is to consider the \emph{minimum
norm interpolator}:
\[
\hat{\w}\coloneqq\arg\min\TNormS{\w}\,\,\,\text{s.t.}\,\,\,\matX\w=\y
\]
 It is still the case that $\hat{\w}=\matX^{\pinv}\y$. Another option
for dealing with non-uniqueness of the minimizer is to add a ridge
term, i.e., add and additive penalty $\lambda\TNormS{\w}$. Let:
\[
\hat{\w}_{\lambda}\coloneqq\arg\min\TNormS{\matX\w-\y}+\lambda\TNormS{\w}
\]
One can show that
\begin{equation}
\hat{\w}_{\lambda}=\mat X_{\lambda}^{+}\y\label{eq:what_lambda}
\end{equation}
where for $\lambda\ge0$ we define: 
\[
\mat X_{\lambda}^{+}\coloneqq\left(\matX^{\T}\matX+\lambda\matI_{d}\right)^{+}\mat X^{\T}
\]
 (see also \cite{bardow2008optimal}). Note that Eq.~\eqref{eq:what_lambda}
holds both for the overparameterized ($d\geq n$) and underparameterized
($d<n$) case.

It holds that the minimum norm interpolator $\hat{\w}$ is equal to
$\hat{\w}_{0}$, and that $\lambda\mapsto\hat{\w}_{\lambda}$ is continuous.
This implies that the various expressions for the expected risk of
$\hat{\w}_{\lambda}$ hold also when $\lambda=0$. So, henceforth
we analyze the expected risk of $\hat{\w}_{\lambda}$ and the results
also apply for $\hat{\w}$.

\subsection{Expected Risk of $\hat{\protect\w}_{\lambda}$}

The following proposition gives an expression for the expected risk
of the regularized estimator $\hat{\w}_{\lambda}$. Note that it holds
both for the overparameterized ($d\geq n$) and underparameterized
($d<n$) case.
\begin{prop}
\label{prop:expected-risk}We have 
\[
\Expect{R(\hat{\w}_{\lambda})}=\underbrace{\TNormS{\mat C_{\rho}^{\nicehalf}\left(\mat I-\Mlambda^{+}\mat M\right)\w}}_{\text{bias}}+\underbrace{\sigma^{2}\Trace{\mat C_{\rho}\Mlambda^{+^{2}}\mat M}}_{\text{variance}}
\]
where 
\[
\Mlambda\coloneqq\matX^{\T}\matX+\lambda\matI_{d}=\matM+\lambda\matI_{d}.
\]
The expectation is with respect to the training noise $\epsilon$.
\end{prop}
This is a standard bias variance decomposition. For completeness the
proof is provided in the Appendix~\ref{app:proofs}.

The last proposition motivates the following design criterion, which
can be viewed as a generalization of classical V-optimality:

\[
\varphi_{\lambda}(\matM)\coloneqq\TNormS{\mat C_{\rho}^{\nicehalf}\left(\mat I-\Mlambda^{+}\mat M\right)\w}+\sigma^{2}\Trace{\mat C_{\rho}\Mlambda^{+^{2}}\mat M}.
\]
For $\lambda=0$ the expression simplifies to the following expression:
\begin{equation}
\varphi_{0}\left(\M\right)=\TNormS{\mat C_{\rho}^{\nicehalf}\left(\mat I-\P_{\matM}\right)\w}+\sigma^{2}\Trace{\mat C_{\rho}\M^{+}}\label{eq:phi_0}
\end{equation}
where $\P_{\matM}=\matM^{\pinv}\matM$ is the projection on the row
space of $\X$. Note that when $n\geq d$ and $\matX$ has full column
rank, $\varphi_{0}(\matM)$ reduces to the variance of underparameterized
linear regression, so minimizing $\varphi_{\lambda}(\matM)$ is indeed
a generalization of the V-optimality criterion.

Note the bias-variance tradeoff in $\varphi_{\lambda}(\matM)$. When
the bias term is much larger than the variance, something we should
expect for small $n$, then it make sense for the design algorithm
to be bias oriented. When the variance is larger, something we should
expect for $n\approx d$ or $n\geq d$, then the design algorithm
should be variance oriented. It is also possible to have mixed nature
in which both bias and variance are of the same order.

\subsection{Practical Criterion}

As is, $\varphi_{\lambda}$ is problematic as an experimental design
criterion since it depends both on $\w$ and on $\matC_{\rho}$. We
discuss how to handle an unknown $\matC_{\rho}$ in Subsection~\ref{subsec:approx-Crho}.
Here we discuss how to handle an unknown $\w.$ Note that obviously
$\w$ is unknown: it is exactly what we want to approximate! If we
have a good guess $\tilde{\w}$ for the true value of $\w$, then
we can replace $\w$ with $\tilde{\w}$ in $\varphi_{\lambda}$. However,
in many cases, such an approximation is not available. Instead, we
suggest to replace the bias component with an upper bound: 
\[
\TNormS{\mat C_{\rho}^{\nicehalf}\left(\mat I-\Mlambda^{+}\mat M\right)\w}\le\TNormS{\w}\cdot\FNormS{\mat C_{\rho}^{\nicehalf}\left(\mat I-\Mlambda^{+}\mat M\right)}.
\]

Let us now define a new design criterion which has an additional parameter
$t\geq0$: 
\[
\bar{\varphi}_{\lambda,t}(\matM)=\underbrace{\FNormS{\mat C_{\rho}^{\nicehalf}\left(\mat I-\Mlambda^{+}\mat M\right)}}_{\text{bias bound (divided by \ensuremath{\TNormS{\w}})}}+\underbrace{t\Trace{\mat C_{\rho}\Mdaggerlambda^{2}\mat M}}_{\text{variance (divided by \ensuremath{\TNormS{\w}})}}.
\]
The parameter $t$ captures an a-priori assumption on the tradeoff
between bias and variance: if we have $t=\sigma^{2}/\TNormS{\w}$,
then 
\[
\varphi_{\lambda}(\matM)\leq\TNormS{\w}\cdot\bar{\varphi}_{\lambda,t}(\matM).
\]
 Thus, minimizing $\bar{\varphi}_{\lambda,t}(\matM)$ corresponds
to minimizing an upper bound of $\varphi_{\lambda}$, if $t$ is set
correctly.

Another interpretation of $\bar{\varphi}_{\lambda,t}(\matM)$ is as
follows. If we assume that $\w\sim{\cal N}(0,\gamma^{2}\matI_{d})$,
then
\[
\ExpectC{\w}{\varphi_{\lambda}(\matM)}=\gamma^{2}\FNormS{\mat C_{\rho}^{\nicehalf}\left(\mat I-\Mlambda^{+}\mat M\right)}+\sigma^{2}\Trace{\mat C_{\rho}\Mlambda^{+^{2}}\mat M}
\]
so if we set $t=\sigma^{2}/\gamma^{2}$ then $\gamma^{2}\bar{\varphi}_{\lambda,t}(\matM)=\ExpectC{\w}{\varphi_{\lambda}(\matM)}$,
so minimizing $\bar{\varphi}_{\lambda,t}(\matM)$ corresponds to minimizing
the expected expected risk if $t$ is set correctly. Again, the parameter
$t$ captures an a-priori assumption on the tradeoff between bias
and variance.
\begin{rem}
One alternative strategy for dealing with the fact that $\w$ is unknown
is to consider a sequential setup where batches are acquired incrementally
based on increasingly refined approximations of $\w$. Such a strategy
falls under the heading of Sequential Experimental Design. In this
paper, we focus on \emph{single shot }experimental design, i.e. examples
are chosen to be labeled once. We leave sequential experimental design
to future research. Although, we decided to focus on the single shot
scenario for simplicity, the single shot scenario actually captures
important real-life scenarios.
\end{rem}

\subsection{\label{subsec:The-relations-of}Comparison to Other Generalized V-Optimality
Criteria}

Consider the case of $\lambda=0$. Note that we can write: 
\[
\bar{\varphi}_{0,t}(\matM)=\FNormS{\mat C_{\rho}^{\nicehalf}\left(\mat I-\matP_{\matM}\right)}+t\Trace{\mat C_{\rho}\M^{+}}.
\]
Recall that the classical V-optimal experimental design criterion
is $\Trace{\mat C_{\rho}\M^{-1}}$, which is only applicable if $n\geq d$
(otherwise, $\matM$ is not invertible). Indeed, if $n\geq d$ and
$\matM$ is invertible, then $\matP_{\matM}=\matI_{d}$ and $\bar{\varphi}_{0,t}(\matM)$
is equal to $\Trace{\mat C_{\rho}\M^{-1}}$ up to a constant factor.
However, $\matM$ is not invertible if $n<d$ and the expression $\Trace{\mat C_{\rho}\M^{-1}}$
does not make sense.

One naive generalization of classical V-optimality for $n<d$ would
be to simply replace the inverse with pseudoinverse, i.e. $\Trace{\mat C_{\rho}\M^{+}}$.
This corresponds to minimizing only the variance term, i.e. taking
$t\to\infty$. This is consistent with classical experimental design
which focuses on variance reduction, and is appropriate when the risk
is variance dominated.

Another generalization of V-optimality can be obtained by replacing
$\matM$ with its regularized (and invertible) version $\matM_{\mu}=\matM+\mu\matI_{d}$
for some chosen $\mu>0$, obtaining $\Trace{\mat C_{\rho}\M_{\mu}^{-1}}$.
This is exactly the strategy employed in transductive experimental
design \cite{yu2006active}, and it also emerges in a Bayesian setup
\cite{chaloner1995bayesian}. One can try to eliminate the parameter
$\mu$ by taking the limit of the minimizers when $\mu\to0$. In order
to describe this limit we need the following definition and theorem:
\begin{defn}
For a family of sets 
\[
\left\{ A_{\lambda}\right\} _{\lambda\in\R},A\subset\R^{d}
\]
 we write 
\[
\overline{\lim}_{\lambda\to\bar{\lambda}}A_{\lambda}=A
\]
 if $\w\in A$ if and only if there exists sequence $\lambda_{n}\to\lambda$
and a sequence $\w_{n}\to\w$ where $\w_{n}\in A_{\lambda_{n}}$ for
sufficiently large $n$.
\end{defn}
\begin{thm}
\label{theorem:argmin_continues}(A restricted version of Theorem
1.17 in \cite{rockafellar2009variational}) Consider $f:\Omega\times\Psi\to\R$
where $\Omega\subseteq\R^{d}$ and $\Psi\subseteq\R$ are compact
and $f$ is continuous. Then
\[
\overline{\lim}_{\lambda\to\bar{\lambda}}\argmin_{\w}f\left(\w,\lambda\right)\subseteq\argmin_{\w}f\left(\w,\bar{\lambda}\right).
\]
\end{thm}
The proof of Theorem~\ref{theorem:argmin_continues} appears in Appendix~\ref{app:proofs}.

The following proposition shows that taking $\mu\to0$ is actually
almost equivalent to taking $t=0$.
\begin{prop}
\label{prop:bias-lambda-limit}For a compact domain $\Omega\subset\R^{d\times d}$
of symmetric positive semidefinite matrices:
\[
\overline{\lim}_{\mu\to0}\argmin_{\M\in\Omega}\Trace{\mat C_{\rho}\mat M_{\mu}^{-1}}\subseteq\argmin_{\M\in\Omega}\Trace{\C_{\rho}\left(\I-\P_{\matM}\right)}.
\]
\end{prop}
\begin{proof}
Consider the function $f(\matM,\mu)=\Trace{\matC_{\rho}\matM_{\mu}^{-1}}$
\[
f(\matM,\mu)=\begin{cases}
\Trace{\mu\C_{\rho}(\matM+\mu\matI)^{-1}} & \mu>0\\
\Trace{\C_{\rho}\left(\I-\matM^{\pinv}\matM\right)} & \mu=0
\end{cases}
\]
defined over $\Omega\times\R_{\geq0}$ where $\R_{\geq0}$ denotes
the set of non-negative real numbers. Note that this function is well-defined
since $\Omega$ is a set of positive semidefinite matrices.

We now show that $f$ is continuous. For $\mu>0$ it is clearly continuous
for every $\matM$, so we focus on the case that $\mu=0$ for an arbitrary
$\matM$. Consider a sequence $\R_{>0}\ni\mu_{n}\to0$ (where $\R_{>0}$
is the set of positive reals) and $\Omega\ni\matM_{n}\to\matM$. Since
$\Omega$ is compact, $\matM\in\Omega$. Let us write a spectral decomposition
of $\matM_{n}$ (recall that $\Omega$ is a set of symmetric matrices)
\[
\matM_{n}=\matU_{n}\Lambda_{n}\matU_{n}^{\T}
\]
where $\Lambda_{n}$ is diagonal with non-negative diagonal elements
(recall that $\Omega$ is a set of positive definite matrices). Let
$\matM=\matU\Lambda\matU^{\T}$ be a spectral decomposition of $\matM$.
Without loss of generality we may assume that $\matU_{n}\to\matU$
and $\Lambda_{n}\to\Lambda$. Now note that 
\[
(\matM_{n}+\mu_{n}\matI)^{-1}\matM_{n}=\matU_{n}(\Lambda_{n}+\mu_{n}\matI)^{-1}\Lambda_{n}\matU_{n}^{\T}
\]
One can easily show that: 
\[
(\Lambda_{n}+\mu_{n}\matI)^{-1}\Lambda_{n}\to\sign(\Lambda)
\]
 where sign is taken entry wise, which implies that: 
\[
(\matM_{n}+\mu_{n}\matI)^{-1}\matM_{n}\to\matU\sign(\Lambda)\matU^{\T}
\]
since matrix multiplication is continuous. Next, note that: 
\[
\matM^{\pinv}\matM=\matU\Lambda^{\pinv}\Lambda\matU^{\T}=\matU\sign(\Lambda)\matU^{\T}
\]
 so, 
\[
(\matM_{n}+\mu_{n}\matI)^{-1}\matM_{n}\to\matM^{\pinv}\matM.
\]
 The Woodbury formula implies that
\[
\mu_{n}\C_{\rho}(\matM_{n}+\mu_{n}\matI)^{-1}=\C_{\rho}\left(\I-(\matM_{n}+\mu_{n}\matI)^{-1}\matM_{n}\right)
\]
so the continuity of the trace operator implies that 
\begin{multline*}
\Trace{\mu_{n}\C_{\rho}(\matM_{n}+\mu_{n}\matI)^{-1}}\\
=\Trace{\C_{\rho}\left(\I-(\matM_{n}+\mu_{n}\matI)^{-1}\matM_{n}\right)}\\
\to\Trace{\C_{\rho}\left(\I-\matM^{\pinv}\matM\right)}
\end{multline*}
 which shows that $f$ is continuous.

Theorem~\ref{theorem:argmin_continues} now implies the claim since
for $\mu>0$ we have
\[
\argmin_{\M\in\Omega}\Trace{\mat C_{\rho}\mat M_{\mu}^{-1}}=\argmin_{\M\in\Omega}\Trace{\mu\mat C_{\rho}\mat M_{\mu}^{-1}}\,.
\]
\end{proof}
We see that the aforementioned generalizations of V-optimality correspond
to either disregarding the bias term ($t=\infty$) or disregarding
the variance term ($t=0$). However, using $\bar{\varphi}_{0,t}(\matM)$
allows much better control over the bias-variance tradeoff (see Figure
\ref{fig:opening}.)

Let us consider now the case of $\lambda>0$. We now show that the
regularized criteria $\Trace{\mat C_{\rho}\mat M_{\mu}^{-1}}$ used
in transductive experimental design (See Proposition~\ref{prop:bias-lambda-limit})
when $\mu=\lambda$ corresponds to also using $t=\lambda$.
\begin{prop}
\label{prop:lambda-lambda}For any matrix space $\Omega$, $\lambda>0$:
\[
\argmin_{\matX\in\Omega}\Trace{\mat C_{\rho}\mat M_{\lambda}^{-1}}=\argmin_{\mat X\in\Omega}\varphibar{\lambda}{\lambda}(\mat M)
\]
\end{prop}
\begin{proof}
Let 
\[
\matA=\left(\matI-(\matM+\lambda\matI_{d})^{-1}\mat M\right)^{2}+\lambda(\Mlambdafull_{d})^{-2}\matM,
\]
 so 
\[
\varphibar{\lambda}{\lambda}(\M)=\Trace{\matC_{\rho}\matA}.
\]
We now have for $\lambda>0$:
\begin{eqnarray*}
\matA & = & \left(\matI-(\matM+\lambda\matI_{d})^{-1}(\mat M+\lambda\matI_{d})+\lambda(\matM+\lambda\matI_{d})^{-1}\right)^{2}\\
 &  & \qquad\qquad\qquad\qquad\qquad\qquad\qquad+\lambda(\Mlambdafull_{d})^{-2}\matM\\
 & = & \lambda^{2}\left(\Mlambdafull_{d}\right)^{-2}+\lambda(\Mlambdafull_{d})^{-2}\mat M\\
 & = & \lambda(\Mlambdafull_{d})^{-2}\left(\mat M+\lambda\mat I_{d}\right)\\
 & = & \lambda\left(\mat M+\lambda\mat I\right)^{-1}
\end{eqnarray*}
so:

\[
\Trace{\matC_{\rho}\matA}=\lambda\Trace{\mat C_{\rho}\left(\mat M+\lambda\mat I\right)^{-1}}.
\]
Since $\lambda>0$ it doesn't affect the minimizer.
\end{proof}
So, transductive experimental design corresponds to a specific choice
of bias-variance tradeoff. Another interesting relation with transductive
experimental design is given by next proposition which is a small
modification of \cite[Theorem 1]{gu2012selective} .
\begin{prop}
\label{prop:transductive-as-bound}For any $\lambda>0$ and $t\ge0$:
\[
\varphibar{\lambda}t(\mat M)\le\left(\lambda+t\right)\Trace{\mat C_{\rho}\mat M_{\lambda}^{-1}}
\]
\end{prop}
In the absence of a decent model of the noise, which is a typical
situation in machine learning, Prop. \ref{prop:transductive-as-bound}
suggests that we can consider simply minimizing $\Trace{\mat C_{\rho}\mat M_{\lambda}^{-1}}$,
and avoid the need to set the additional parameter $t$. However,
this approach may be suboptimal in overparameterized regimes, as it
implicitly considers $t=\lambda$ (see Prop. \ref{prop:lambda-lambda}).
In a bias dominated regime this can put too much emphasis on minimizing
the variance. A sequential approach for experimental design can lead
to better modeling of the noise, thereby assisting in dynamically
setting $t$ during acquisition-learning cycles. However, in a single
shot regime, noise estimation is difficult. Arguably, there exists
better values for $t$ than using a default rule-of-thumb $t=\lambda$.
In particular, we conjecture that $t=0$ is a better rule-of-thumb
then $t=\lambda$ for severely overparameterized regimes as it suppresses
the potential damage of choosing a too large $\lambda$ and it is
reasonable also if $\lambda$ is small (since anyway we are in a bias
dominated regime), so we can focus on minimizing the bias only. In
Section~\ref{sec:Experiments} we report an experiment that supports
this conjecture. Notice that $t=\infty$ corresponds to minimizing
the variance, while $t=0$ corresponds to minimizing the bias.

\subsection{\label{subsec:approx-Crho}Approximating $\protect\matC_{\rho}$}

Our criteria so far depended on $\matC_{\rho}$. Oftentimes $\matC_{\rho}$
is unknown. However, it can be approximated using unlabeled data.
Suppose we have $m$ unlabeled points (i.e. drawn form $\rho$), and
suppose we write them as the rows of $\matV\in\R^{m\times d}$. Then
$\Expect{m^{-1}\matV^{\T}\matV}=\matC_{\rho}$. Thus, for $\lambda\ge0$,
we can write
\begin{multline*}
m\varphilambda(\matM)\approx\psi_{\lambda}(\matM)\\
\coloneqq\TNormS{\V\left(\mat I_{d}-\Mlambda^{+}\mat M\right)\w}+\sigma^{2}\Trace{\V\Mlambda^{+^{2}}\mat M\V^{\T}}.
\end{multline*}
and use $\psi_{\lambda}(\matM)$ instead of $\varphi_{\lambda}(\matM)$.
For minimum norm interpolation we have 
\[
\psi_{0}(\matM)=\TNormS{\V\left(\mat I_{d}-\P_{\matM}\right)\w}+\sigma^{2}\Trace{\V\mat M^{+}\V^{\T}}.
\]
Again, let us turn this into a practical design criteria by introducing
an additional parameter $t$: 
\begin{equation}
\bar{\psi}_{\lambda,t}(\matM)\coloneqq\FNormS{\V\left(\mat I_{d}-\Mlambda^{+}\mat M\right)}+t\Trace{\V\Mlambda^{+^{2}}\mat M\V^{\T}}.\label{eq:psibar_lambda-1}
\end{equation}

\section{\label{sec:pool}Pool-based Overparameterized Experimental Design}

In the previous section we defined design criteria $\bar{\varphi}_{\lambda,t}$
and $\bar{\psi}_{\lambda,t}$ which are appropriate for overparameterized
linear regression. While one can envision a situation in which such
we are free to choose $\matX$ so to minimize the design criteria,
in much more realistic pool-based active learning we assume that we
are given in advance a large pool of unlabeled data $\x_{1},\dots,\x_{m}$.
The training set is chosen to be a subset of the pool. This subset
is then labeled, and learning performed. The goal of pool-based experimental
design algorithms is to chose the subset to be labeled.

We formalize the pool-based setup as follows. Recall that to approximate
$\matC_{\rho}$ we assumed we have a pool of unlabeled data written
as the rows of $\matV\in\R^{m\times d}$. We assume that $\matV$
serves also as the pool of samples from which $\matX$ is selected.
For a matrix $\matA$ and index sets $\calS\subseteq[n]$, $\calT\subseteq[d]$,
let $\matA_{\calS,\calT}$ be the matrix obtained by restricting to
the rows whose index is in $\calS$ and the columns whose index is
in $\calT$. If $:$ appears instead of an index set, that denotes
the full index set corresponding to that dimension. Our goal is to
select a subset ${\cal S}$ of cardinality $n$ such that $\bar{\psi}_{\lambda,t}(\matV_{{\cal S},:}^{\T}\matV_{{\cal S},:})$
is minimized (i.e., setting $\matX=\matV_{{\cal S},:}$). Formally,
we pose following problem:
\begin{problem}
\label{prob:pool}(Pool-based Overparameterized V-Optimal Design)
Given a pool of unlabeled examples $\matV\in\R^{m\times d}$, a regularization
parameter $\lambda\geq0$, a bias-variance tradeoff parameter $t\geq0$,
and a design size $n$, find a minimizer of 
\[
\min_{{\cal S}\subseteq[m],\,|{\cal S}|=n}\bar{\psi}_{\lambda,t}(\matV_{{\cal S},:}^{\T}\matV_{{\cal S},:}).
\]
\end{problem}
Problem~\ref{prob:pool} is a generalization of the Column Subset
Selection Problem (CSSP)~\cite{boutsidis2008improved}. In the CSSP,
we are given matrix $\matU\in\R^{d\times m}$ and target number of
columns $n$, and our goal is to select a subset ${\cal T}$ which
is a minimizer of
\[
\min_{{\cal T}\subseteq[m],\,|{\cal T}|=n}\FNormS{(\matI_{d}-\matU_{:,{\cal T}}\matU_{:,{\cal T}}^{\pinv})\matU}
\]
When $\lambda=0$ and $t=0$, Problem~\ref{prob:pool} reduces to
the CSSP for $\matU=\matV^{\T}$. 

\section{\label{sec:coresets}Relation to Coresets}

In this section we show that the $\lambda=t=0$ case is also related
to the \emph{coreset} approach for active learning~ \cite{sener2017active,pinsler2019bayesian,ash2019deep,geifman2017deep}.

The idea in the coreset approach for active learning is to find a
${\cal S}$ which minimizes the \emph{corset loss:}
\[
C\mathrm{({\cal S})}\coloneqq\left|\frac{1}{m}\sum_{i=1}^{m}l(\x_{i},y_{i}\,|\,\calS)-\frac{1}{|\calS|}\sum_{i\in\calS}l(\x_{i},y_{i}\,|\,{\cal S})\right|.
\]
In the above $l(\x,y\,|\,{\cal S})$ is a loss function, and the conditioning
on $\calS$ denotes that the parameters of the loss function are the
ones obtained when training only using indices selected in $\calS$.

One popular approach to active learning using coresets is to find
a \emph{coverset. }A $\delta$-coverset of a set of points ${\cal A}$
is a set of points ${\cal B}$ such that for every $\x\in{\cal A}$
there exists a $\y\in{\cal B}$ such that $\TNorm{\x-\y}\leq\delta$
(other metrics can be used as well). Sener and Savarese showed that
under a certain Lipschitz and boundness conditions on the loss function
and the regression function, if $\{\x_{i}\}_{i\in{\cal S}}$ is a
$\delta$-coverset of $\{\x_{i}\}_{i\in[m]}$ then 
\[
C({\cal S})\leq O(\delta+m^{-\nicehalf})
\]
which motivates finding a $\calS$ that minimizes $\delta_{\calS}$,
where $\delta_{\calS}$ denotes the minimal $\delta$ for which $\{\x_{i}\}_{i\in{\cal S}}$
is a $\delta$-coverset of $\{\x_{i}\}_{i\in[m]}$~\cite{sener2017active}
.

Since for a $\x$ in the training set (which is a row of $\matV$)
it holds that $\TNormS{\x(\matI_{d}-\matP_{\matM})}$, for $\matM=\matV_{{\cal S},:}^{\T}\matV_{{\cal S},:}$,
is the minimal distance from $\x$ to the span of $\{\x_{i}\}_{i\in{\cal S}}$,
and as such that distance is always smaller than the distance between
$\x$ and it's closest point in $\{\x_{i}\}_{i\in{\cal S}}$, it is
easy to show that 
\[
n^{-1}\bar{\psi}_{0,0}(\matV_{{\cal S},:}^{\T}\matV_{{\cal S},:})\leq\delta_{\calS}^{2}.
\]
Thus, minimizing $\delta_{\calS}$ can be viewed as minimizing an
upper bound on the bias term when $\lambda=0$.

\section{Optimization Algorithm}

In this section we propose an algorithm for overparameterized experimental
design. Our algorithm is based on greedy minimization of a kernalized
version of $\bar{\psi}_{\lambda,t}(\matV_{{\cal S},:}^{\T}\matV_{{\cal S},:})$.
Thus, before presenting our algorithm, we show how to handle feature
spaces defined by a kernel. 

\subsection{Kernelization}

If $|\calS|\leq d$ and $\matV_{{\cal S},:}$ has full row rank we
have $(\matV_{{\cal S},:})_{\lambda}^{\pinv}=\matV_{{\cal S},:}^{\T}\left(\matV_{{\cal S},:}\matV_{{\cal S},:}^{\T}+\lambda\I_{|{\cal S}|}\right)^{-1}$
which allows us to write {\small{}
\begin{multline*}
\text{\ensuremath{\bar{\psi}_{\lambda,t}}}(\matV_{{\cal S},:}^{\T}\matV_{{\cal S},:})=\\
\Trace{\V\left[\I-2\matV_{{\cal S},:}^{\T}\left(\matV_{{\cal S},:}\matV_{{\cal S},:}^{\T}+\lambda\I_{|{\cal S}|}\right)^{-1}\matV_{{\cal S},:}\right]\V^{\T}}\\
+\mathbf{Tr}\left(\matV\matV_{{\cal S},:}^{\T}\left(\matV_{{\cal S},:}\matV_{{\cal S},:}^{\T}+\lambda\I_{|{\cal S}|}\right)^{-1}\matV_{{\cal S},:}\right.\\
\left.\matV_{{\cal S},:}^{\T}\left(\matV_{{\cal S},:}\matV_{{\cal S},:}^{\T}+\lambda\I_{|{\cal S}|}\right)^{-1}\matV_{{\cal S},:}\matV^{\T}\right)\\
+t\Trace{\V\matV_{{\cal S},:}^{\T}\left(\matV_{{\cal S},:}\matV_{{\cal S},:}^{\T}+\lambda\I_{|{\cal S}|}\right)^{-2}\matV_{{\cal S},:}\V^{\T}}
\end{multline*}
}Let now $\matK\coloneqq\matV\matV^{\T}\in\R^{m\times m}$. Then $\matV_{{\cal S},:}\matV_{{\cal S},:}^{\T}=\matK_{\calS,{\cal S}}$
and $\V\matV_{{\cal S},:}^{\T}=\matK_{:,{\cal S}}$ . Since $\Trace{\matK}$
is constant, minimizing $\text{\ensuremath{\bar{\psi}_{\lambda,t}}}(\matV_{{\cal S},:}^{\T}\matV_{{\cal S},:})$
is equivalent to minimizing{\small{}
\begin{multline}
J_{\lambda,t}({\cal S})\coloneqq\mathbf{Tr}\bigg(\K_{:,\calS}\bigg[\\
\left(\matK_{\calS,{\cal S}}+\lambda\I_{|{\cal S}|}\right)^{-1}\left(-2\I_{|{\cal S}|}+\matK_{\calS,{\cal S}}\left(\matK_{\calS,{\cal S}}+\lambda\I_{|{\cal S}|}\right)^{-1}\right)\\
+t\left(\matK_{\calS,{\cal S}}+\lambda\I_{|{\cal S}|}\right)^{-2}\bigg]\K_{:,\calS}^{\T}\bigg).\label{eq:J_lt}
\end{multline}
}For $\lambda=0$ we have a simpler form: 
\[
J_{0,t}({\cal S})=\Trace{\K_{:,\calS}\left[-\matK_{\calS,{\cal S}}^{-1}+t\matK_{\calS,{\cal S}}^{-2}\right]\K_{:,\calS}^{\T}}.
\]

Interestingly, when $\lambda=0$ and $t=0$, minimizing $J_{0,0}({\cal S})$
is equivalent to maximizing the trace of the Nystrom approximation
of $\matK$. Another case for which Eq.~\eqref{eq:J_lt} simplifies
is $t=\lambda$ (this equation was already derived in~\cite{yu2006active}):

\[
J_{\lambda,\lambda}({\cal S})=\Trace{-\K_{:,\calS}\left(\matK_{\calS,{\cal S}}+\lambda\I_{|{\cal S}|}\right)^{-1}\K_{:,\calS}^{\T}}.
\]

Eq.~\eqref{eq:J_lt} allows us, via the kernel trick, to perform
experimental design for learning of nonlinear models defined using
high dimensional feature maps. Denote our unlabeled pool of data by
$\z_{1},\dots,\z_{m}\in\R^{D}$, and that we are using a feature map
$\phi:\R^{d}\to{\cal H}$ where ${\cal H}$ is some Hilbert space
(e.g., $\calH=\R^{d}$), i.e. the regression function is $y(\z)=\dotprodH{\phi(\z)}{\w}$.
We can then envision the pool of data to be defined by $\x_{j}=\phi(\z_{j})$,~$j=1,\dots,m$.
If we assume we have a kernel function $k:\R^{D}\times\R^{D}\to\R^{D}$
such that $k(\x,\z)=\dotprodH{\phi(\x)}{\phi(\z)}$ then $J_{\lambda,t}(\calS)$
can be computed without actually forming $\x_{1},\dots,\x_{m}$ since
entries in $\matK$ can be computed via $k$. If ${\cal H}$ is the
Reproducing Kernel Hilbert Space of $k$ then this is exactly the
setting that corresponds to kernel ridge regression (possibly with
a zero ridge term).

\subsection{\label{subsec:greedy-algorithm}Greedy Algorithm}

We now propose our algorithm for overparameterized experimental design,
which is based on greedy minimization of $J_{\lambda,t}({\cal S})$.
Greedy algorithms have already been shown to be effective for classical
experimental design \cite{yu2006active,avron2013faster,chamon2017approximate},
and it is reasonable to assume this carries on to the overparameterized
case.

Our greedy algorithm proceeds as follows. We start with ${\cal S}^{(0)}=\emptyset$,
and proceed in iteration. At iteration $j$, given selected samples
${\cal S}^{(j-1)}\subset[m]$ the greedy algorithm finds the index
$i^{(j)}\in[m]-{\cal S}^{(j-1)}$ that minimizes $J_{\lambda,t}\left({\cal S}^{(j-1)}\cup\{i^{(j)}\}\right).$
We set $\calS^{(j)}\gets{\cal S}^{(j-1)}\cup\{i^{(j)}\}$. We continue
iterating until ${\cal S}^{(j)}$ reaches its target size and/or $J_{\lambda,t}({\cal S})$
is small enough.

The cost of iteration $j$ in a naive implementation is $O\left(\left(m-j\right)\left(mj^{2}+j^{3}\right)\right)$.
Through careful matrix algebra, the cost of iteration $j$ can be
reduced to $O((m-j)(mj+j^{2}))=O(m^{2}j)$ (since $j\leq m$). The
cost of finding a design of size $n$ is then $O(m^{2}(n^{2}+D))$
assuming the entire kernel matrix $\matK$ is formed at the start
and a single evaluation of $k$ takes $O(D)$. Details are delegated
to Appendix~\ref{sec:alg-details}.

\section{Single Shot Deep Active Learning\label{sec:Single-shot-active-deep}}

There are few ways in which our proposed experimental design algorithm
can be used in the context of deep learning. For example, one can
consider a sequential setting where current labeled data are used
to create a linear approximation via the Fisher information matrix
at the point of minimum loss~\cite{sourati2018active}. However,
such a strategy falls under the heading of Sequential Experimental
Design, and, as we previously stated, in this paper we focus on single
shot active learning, i.e. no labeled data is given neither before
acquisition nor during acquisition \cite{yang2019single}.

In order to design an algorithm for deep active learning, we leverage
a recent breakthrough in theoretical analysis of deep learning - the
Neural Tangent Kernel (NTK) \cite{jacot2018neural,lee2019wide,arora2019exact}.
A rigorous exposition of the NTK is beyond the scope of this paper,
but a short and heuristic explanation is sufficient for our needs.

Consider a DNN, and suppose the weights of the various layers can
be represented in a vector $\vec{\theta}\in\R^{d}$. Given a specific
$\vec{\theta}$, let $f_{\vec{\theta}}(\cdot)$ denote the function
instantiated by that network when the weights are set to $\vec{\theta}$.
The crucial observation is that when the network is wide (width in
convolutional layers refers to the number of output channels) enough,
we use a quadratic loss function (i.e., $l(f_{\vec{\theta}}(\x),y)=\nicehalf(f_{\vec{\theta}}(\x)-y)^{2}$),
and the initial weights $\vec{\theta}_{0}$ are initialized randomly
in a standard way, then when training the DNN using gradient descent,
the vector of parameters $\vec{\theta}$ stays almost fixed. Thus,
when we consider $\vec{\theta}_{1},\vec{\theta}_{2},\dots$ formed
by training, a first-order Taylor approximation is:
\[
f_{\vec{\theta}_{k}}(\x)\approx f_{\vec{\theta}_{0}}(\x)+\nabla_{\vec{\theta}}f_{\vec{\theta_{0}}}(\x)^{\T}(\vec{\theta}_{k}-\vec{\theta}_{0})
\]
Informally speaking, the approximation becomes an equality in the
infinite width limit. The Taylor approximation implies that if we
further assume that $\vec{\theta}_{0}$ is such that $f_{\vec{\theta}_{0}}(\x)=0$,
the learned prediction function of the DNN is well approximated by
the solution of a kernel regression problem with the (Finite) Neural
Tangent Kernel, defined as
\[
k_{f,\vec{\theta}_{0}}(\x,\z)\coloneqq\nabla_{\vec{\theta}}f_{\vec{\theta_{0}}}(\x)^{\T}\nabla_{\vec{\theta}}f_{\vec{\theta_{0}}}(\z)
\]
We remark that there are few simple tricks to fulfill the requirement
that $f_{\vec{\theta}_{0}}(\x)=0$. 

It has also been shown that under certain initialization distribution,
when the width goes to infinity, the NTK $k_{f,\vec{\theta}_{0}}$
converges in probability to a deterministic kernel $k_{f}$ - the
\emph{infinite NTK}. Thus, in a sense, instead of training a DNN on
a finite width network, we can take the width to infinity and solve
a kernel regression problem instead. 

Although, it is unclear whether the infinite NTK can be an effective
alternative to DNNs in the context of inference, one can postulate
that it can be used for deep active learning. That is, in order to
select examples to be labeled, use an experimental design algorithm
for kernel learning applied to the corresponding NTK. Specifically,
for single shot deep active learning, we propose to apply the algorithm
presented in the previous section to the infinite NTK. In the next
section we present preliminary experiments with this algorithm. We
leave theoretical analysis to future research.

\section{Empirical Evaluation \label{sec:Experiments}}

\subsection{Experimental Parameters Exploration and Comparison to Transductive
Experimental Design}

\begin{figure*}
\noindent\begin{minipage}[t]{1\columnwidth}%
\begin{center}
\includegraphics[width=1\textwidth]{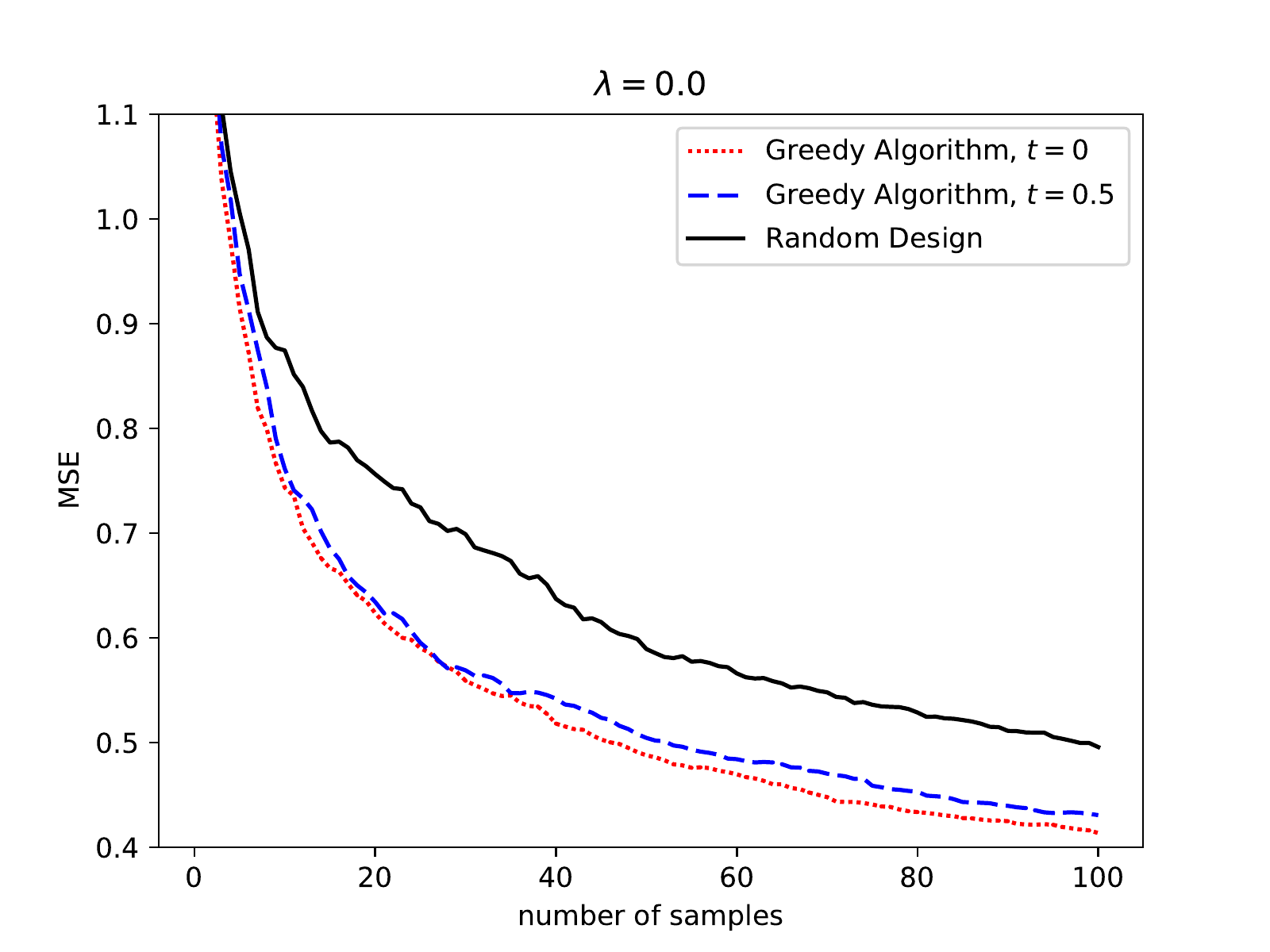}
\par\end{center}%
\end{minipage}\hspace*{\fill}%
\noindent\begin{minipage}[t]{1\columnwidth}%
\begin{center}
\includegraphics[width=1\textwidth]{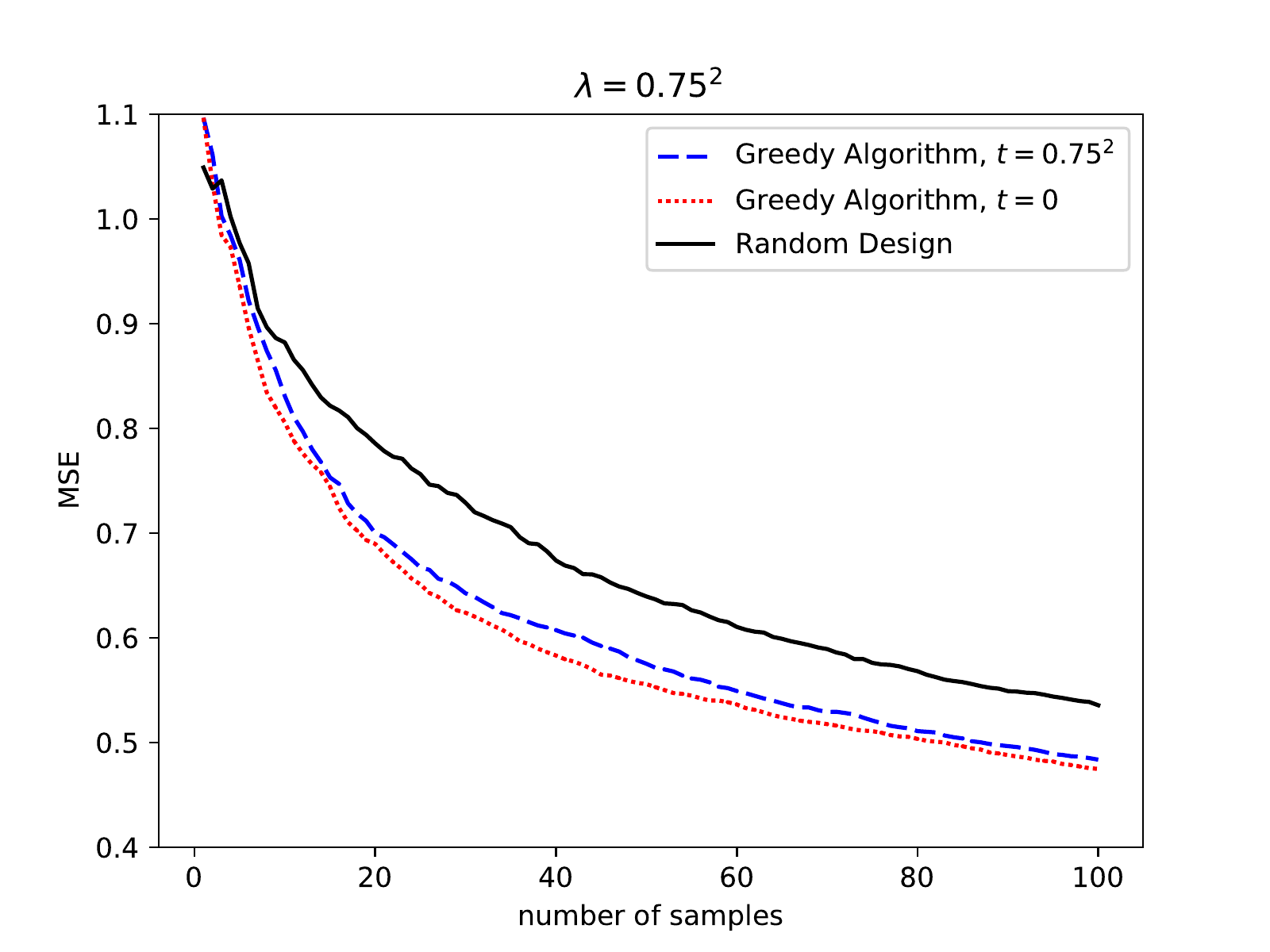}
\par\end{center}%
\end{minipage}
\centering{}\caption{\label{fig:Optimal-design-for}Kernel regression experiments on MNIST.}
\end{figure*}

\begin{wrapfigure}{i}{2\columnwidth}%
\end{wrapfigure}%

In this subsection we report a set of experiments on a kernel ridge
regression setup (though in one experiment we set the ridge term to
$0$, so we are using interpolation). We use the MNIST handwriting
dataset \cite{lecun2010mnist}, where the regression target response
was computed by applying one-hot function on the labels 0-9. Nevertheless,
we still measure the MSE, and do not use the learnt models as classifiers.
We use the RBF kernel $k(\x,\z)=\exp(-\gamma\TNormS{\x-\z})$ with
parameter $\gamma=\nicefrac{1}{784}$. From the dataset, we used the
standard test set of 10000 images and selected randomly another 10000
images from the rest of the 60000 images as a pool. We used our proposed
greedy algorithm to select a training set of sizes $1$ to $100$.
We use two values of $\lambda$: $\lambda=0$ (interpolation), and
$\lambda=0.75^{2}$. The optimal $\lambda$ according to cross validation
was the smallest we checked so we just used $\lambda=0$. However,
in some cases having a $\lambda>0$ is desirable from a computational
perspective, e.g. it caps the condition number of the kernel matrix,
making the linear system easier to solve. Furthermore, in real world
scenarios, oftentimes we do not have any data before we start to acquire
labels, and if we do, it is not always distributed as in the test
data, so computing the optimal $\lambda$ can be a challenging.

Results are reported in Figure~\ref{fig:Optimal-design-for}. The
left panel show the results for $\lambda=0$. We report results for
$t=0$ and $t=0.5$. The choice of $t=0$ worked better. Kernel models
with the RBF kernel are highly overparameterized (the hypothesis space
is infinite dimensional), so we expect the MSE to be bias dominated,
in which case a small $t$ (or $t=0$) might work best. Recall that
the option of $\lambda=t=0$ is equivalent to the Column Subset Selection
Problem, is the limit case of transductive experimental design~\cite{yu2006active},
and can be related to the coreset approach (specifically \cite{sener2017active}).

The case of $\lambda=0.75^{2}$ is reported in the right panel of
Figure~\ref{fig:Optimal-design-for}. We tried $t=0$ and $t=\lambda=0.75^{2}$.
Here too, using a purely bias oriented objective (i.e., $t=0$) worked
better. Note that this is in contrast with classical OED which use
variance oriented objectives. The choice of $t=\lambda$ worked well,
but not optimally. In general, in the reported experiments, and other
experiments conducted but not reported, it seems that the choice of
$t=\lambda$, which is, as we have shown in this paper, equivalent
to transductive experimental design, usually works well, but is not
optimal.

\subsection{Transductive vs $\bar{\psi}_{\lambda,0}$ Criterion (i.e., variance-oriented
vs. bias-oriented designs)}

\begin{table}
\caption{\label{tab:t_eq_0-vs}$\bar{\psi}_{\lambda,0}$ vs $\text{\ensuremath{\bar{\psi}_{\lambda,\lambda}}}$
on UCI datasets. We generated designs on 112 classification datasets.
Each cell details the number of datasets in which that selection of
$t$ was clearly superior to the other possible choice, or the same
(for the \textquotedblleft SAME\textquotedblright{} column).}

\vphantom{}
\centering{}%
\begin{tabular}{>{\centering}p{4em}>{\centering}p{4em}>{\centering}p{4em}>{\centering}p{4em}}
\textbf{$\lambda$} & \textbf{$\bar{\psi}_{\lambda,\lambda}$ is better} & \textbf{$\bar{\psi}_{\lambda,0}$ is better} & \textbf{SAME}\tabularnewline
\hline 
 &  &  & \tabularnewline
0.001 & 5 & 8 & 99\tabularnewline
0.01 & 7 & 9 & 96\tabularnewline
0.1 & 16 & 16 & 80\tabularnewline
1.0 & 21 & \textcolor{red}{43} & 48\tabularnewline
10.0 & 19 & \centering{}\textcolor{red}{68} & 25\tabularnewline
\end{tabular}
\end{table}

$\bar{\psi}_{\lambda,0}$ and $\bar{\psi}_{\lambda,\lambda}$ are
simplified version of $\bar{\psi}_{\lambda,t}$ criterion. Our conjecture
is that in the overparameterized regime $\bar{\psi}_{\lambda,0}$
is preferable, at least for relatively large $\lambda$. Table \ref{tab:t_eq_0-vs}
empirically supports our conjecture. In this experiment, we performed
an experimental design task on 112 classification datasets from UCI
database (similar to the list that was used by \cite{arora2019harnessing}
). Learning is performed using kernel ridge regression with standard
RBF kernel. We tried different values of $\lambda$ and checked which
criterion brings to a smaller classification error on a test set when
selecting 50 samples. Each entry in Table \ref{tab:t_eq_0-vs} counts
how many times $\bar{\psi}_{\lambda,\lambda}$ , won $\bar{\psi}_{\lambda,0}$
won or the error was the same. We consider an equal error when the
difference is less the 5\%.

\subsection{Deep Active Learning}

\begin{figure}
\noindent\begin{minipage}[c]{1\columnwidth}%
\begin{center}
\includegraphics[width=1\textwidth]{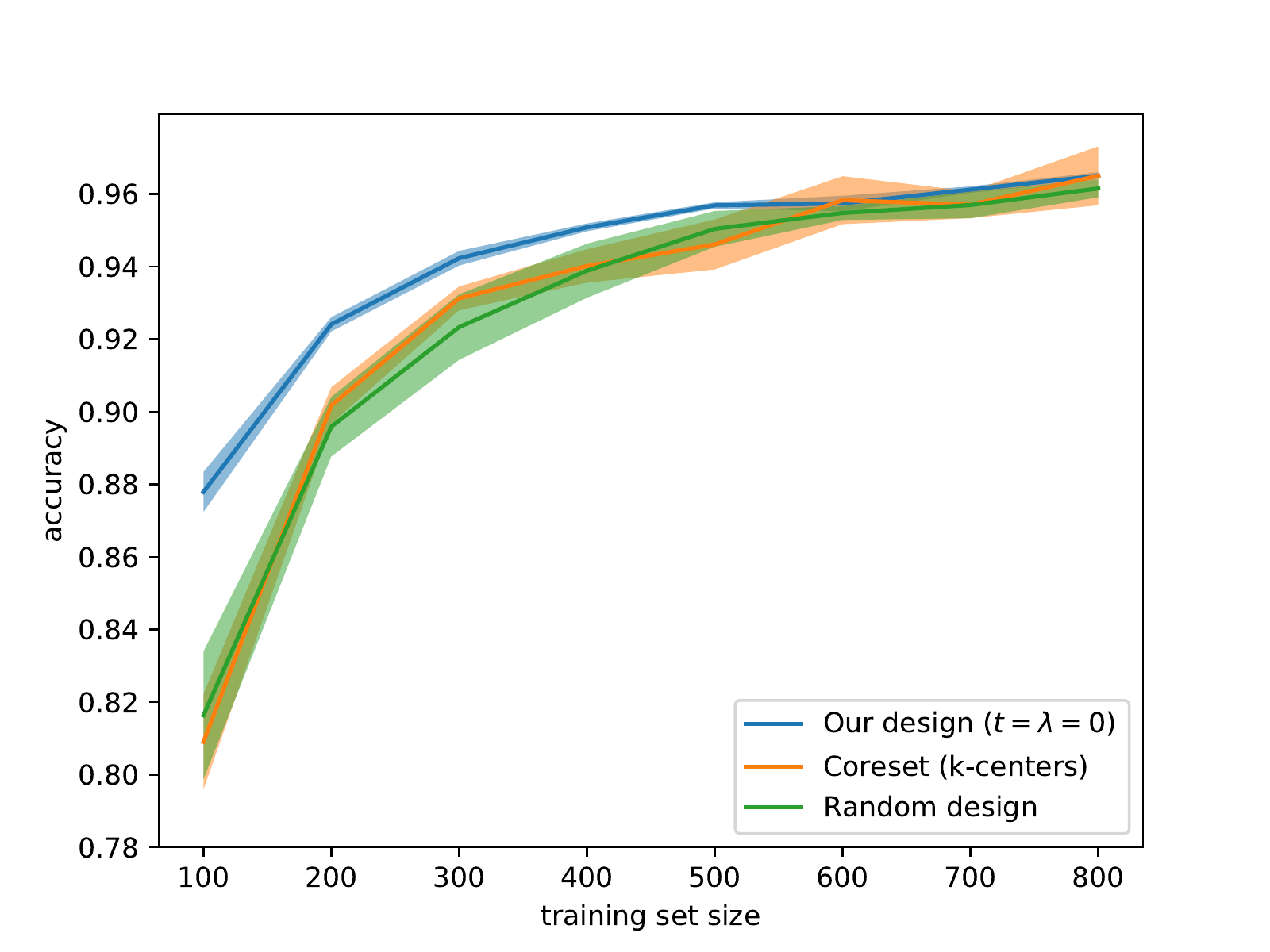}
\par\end{center}%
\end{minipage}

\caption{\label{fig:optimal-design-with-wide-lenet}Single shot active learning
with Wide-LeNet5 model on MNIST.}
\end{figure}

Here we report preliminary experiments with the proposed algorithm
for single shot deep active learning (Section~\ref{sec:Single-shot-active-deep}).
We used the MNIST dataset, and used the square loss for training.
As for the network architecture, we used a version of LeNet5 \cite{lecun1998gradient}
that is widen by a factor of 8. we refer to this network as ``Wide-LeNet5''.

The setup is as follows. We use Google's open source neural tangents
library \cite{neuraltangents2020} to compute Gram matrix of the infinite
NTK using 59,940 training samples (we did not use the full 60,000
training samples due to batching related technical issues). We then
incrementally selected greedy designs of up to 800 samples using three
methods: (a) the algorithm proposed in Section~\ref{subsec:greedy-algorithm}
withparameters to $\lambda=t=0$; (b) \emph{k-centers algorithm} \cite{wolf2011facility}
that finds an approximately optimal coverset (coreset)\footnote{To achieve a reasonable result with the k-centers algorithm we needed
to replace the greedy selection of the next sample according to the
maximum score with a random sampling according to probability proportional
to the score.}; (c) random selection. We now trained the original neural network
with different design sizes, each design with five different random
initial parameters. Learning was conducted using SGD, with fixed learning
rate of 0.1, batch size of 128, and no weight decay. Instead of counting
epochs, we simply capped the number of SGD iterations to be equivalent
to $20$ epochs of the full training set. We computed the accuracy
of the model predictions on 9963 test-set samples (again, due to technical
issues related to batching).

Figure \ref{fig:optimal-design-with-wide-lenet} reports the mean
and standard deviation (over the parameters initialization) of the
final accuracy. We see a consistent advantage in terms of accuracy
for designs selected via our algorithm, though as expected the advantage
shrinks as the training size increase. Notice, that comparing the
accuracy of our design with 400 training samples, random selection
required as many as 600 for Wide-LeNet5 to achieve the same accuracy.

\textcolor{black}{Two remarks are in order. First, to prevent overfitting
and reduce computational load, at each iteration of the greedy algorithm
we computed the score for only on a subset of 2000 samples from the
pool. Second, to keep the experiment simple we refrained from using
mechanisms that ensure $f_{\vec{\theta}_{0}}=0$.}

\subsection{Experiment: Single Shot Active Learning for Narrow Networks}

In Figure \ref{fig:Wide-LeNet5-vs-LeNet5} we compare the result of
our method on LeNet5 with the result of our method on Wide-LeNet5.
We see that while the result on the wide version are generally better,
both for random designs and our design, our method brings a consistent
advantage over random design. In both the narrow and the wide versions
it requires about 600 training samples for the random design to achieve
the accuracy achieved using our algorithm with only 400 training samples.

The parameters used by our algorithm to select the design are $\lambda=t=0$.
For the network training we used SGD with batch size 128, leaning
rate 0.1 and no regularization. The SGD number iterations is equivalent
to $20$ epochs of the full training set.

\begin{figure}
\noindent\begin{minipage}[t]{1\columnwidth}%
\begin{center}
\includegraphics[width=1\columnwidth]{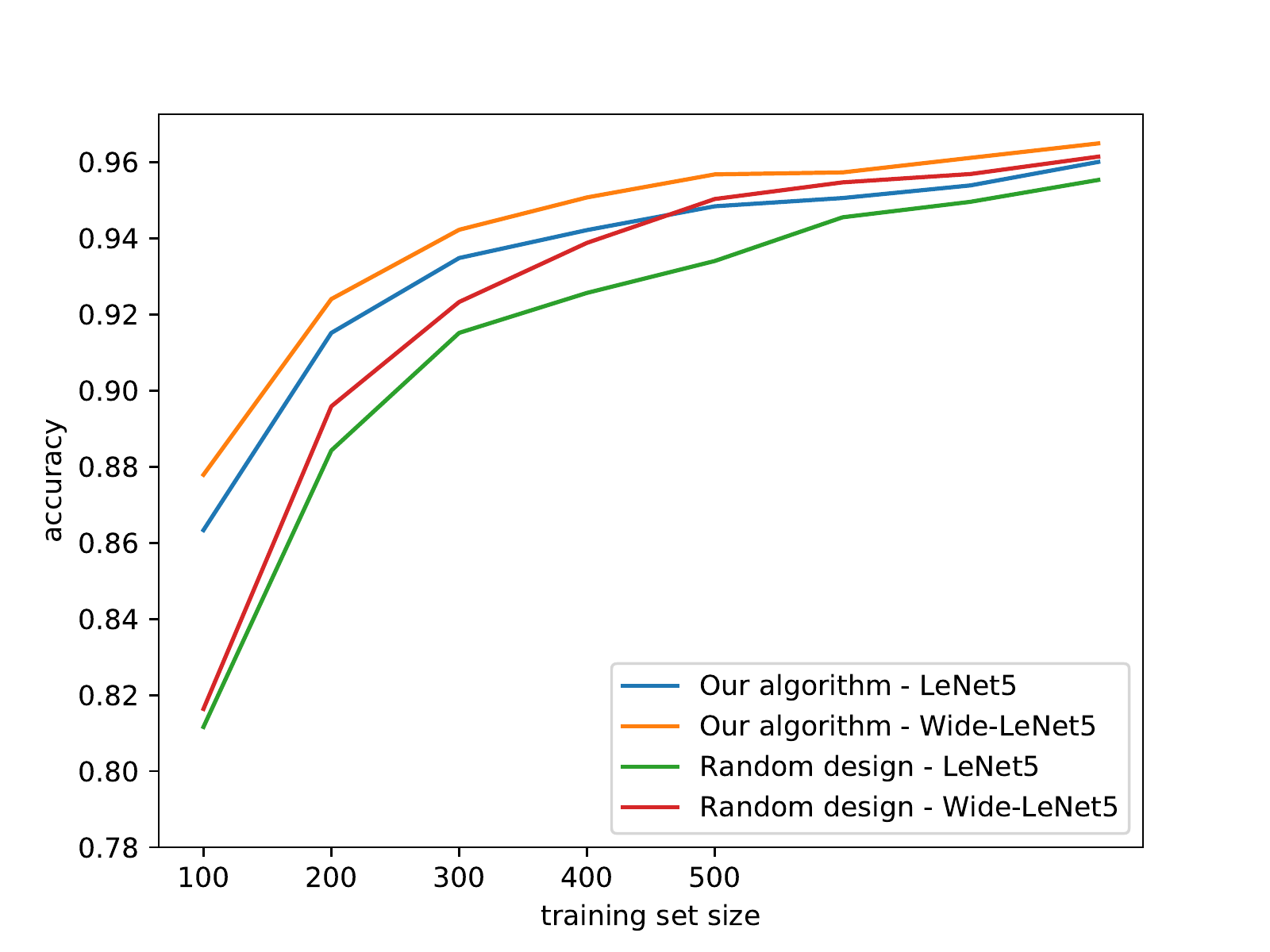}
\par\end{center}%
\end{minipage}

\caption{Wide-LeNet5 vs LeNet5 \label{fig:Wide-LeNet5-vs-LeNet5}}
\end{figure}

\subsection{Sequential vs Single Shot Active Learning}

While in this work focus on the single shot active learning, an interesting
question is how does it compare to sequential active learning. In
sequential active learning we alternate between a model improving
step and a step of new labels acquisition,. This obviously gives an
advantage to sequential active learning over single shot active learning,
as the latter is a restricted instance of the former.

As we still do not have a sequential version of our algorithm to compare
with, we chose to experimentally compare our single shot algorithm
with the classical method of \emph{uncertainty sampling} \cite{bartlett2019benign}.
This method has proved to be relatively efficient for neural networks
\cite{gal2017deep}. Uncertainty sampling based active learning requires
computing the uncertainty of the updated model regarding each sample
in the pool. As such, this approach is sequential by nature.

Usually uncertainty sampling is derived in connection to the cross
entropy since in that case the network output after the $\mathrm{softmax}$
layer can be interpreted as a probability estimation of $y=i$ given
$\x$, which we symbolize as $p_{i}(\x)$. The uncertainty score (in
one common version) is then given by 
\[
1-\max_{i\in[L]}p_{i}(\x).
\]
Because we use the square lose, we need to make some adaptation for
the way of $p_{i}(\x)$ is computed. Considering the fact that the
square loss is an outcome of a \emph{maximum likelihood }model that
given $\x$ assumes $\y\sim\Normal(f(\x),\I_{L})$, it make sense
to use 
\[
p_{i}(\x)=(2\pi)^{-\frac{L}{2}}e^{-\frac{1}{2}\TNormS{\y_{i}-f(\x)}},
\]
where $\y_{i}$ is the $\mathrm{onehot}$ vector of $i$.

\begin{figure}
\noindent\begin{minipage}[t]{1\columnwidth}%
\begin{center}
\includegraphics[width=1\columnwidth]{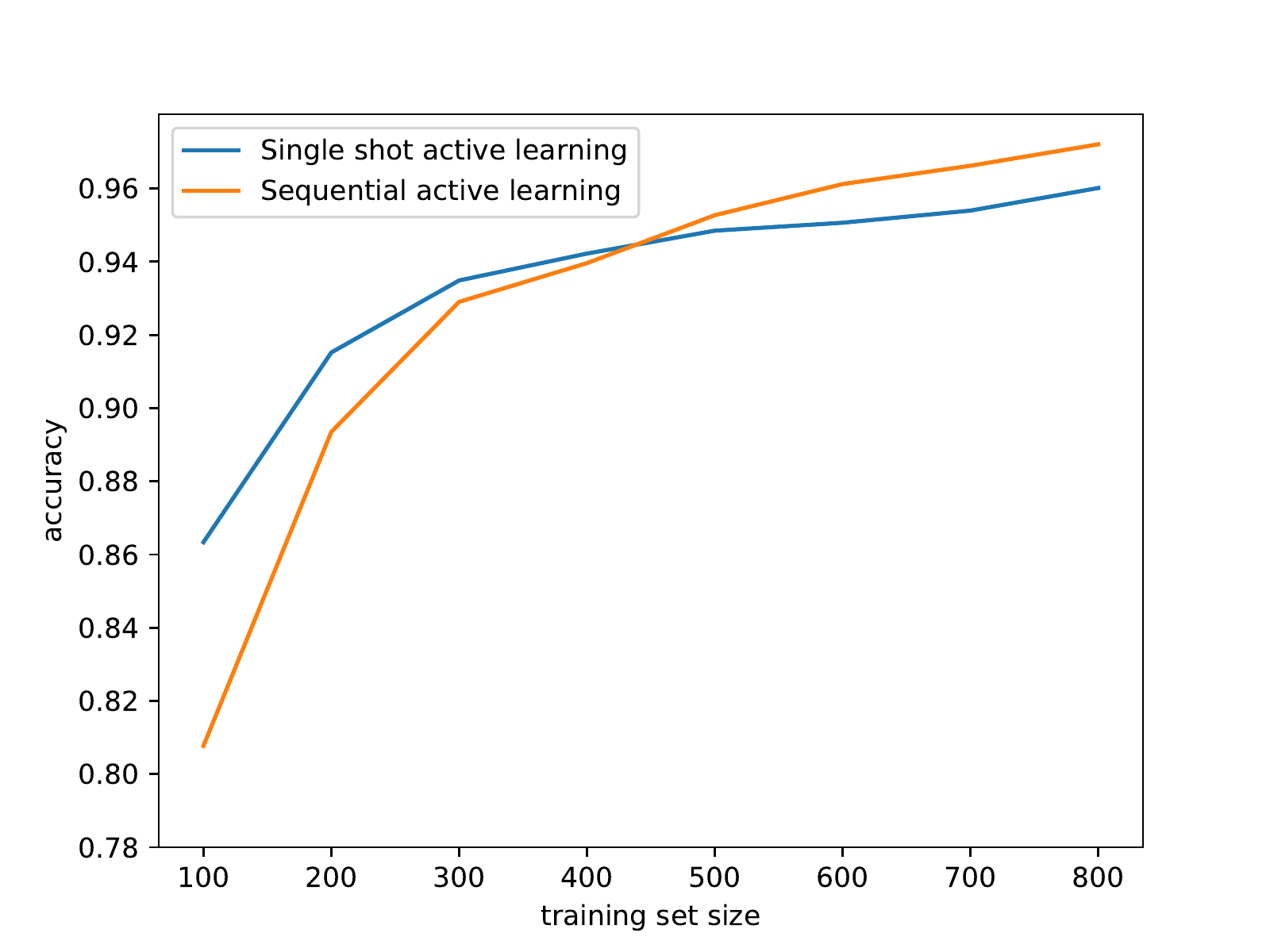}
\par\end{center}%
\end{minipage}

\caption{Single shot active learning vs sequential active learning. MNIST and
(standard) LeNet5\label{fig:single-shot_vs_sequential}}
\end{figure}

Figure \ref{fig:single-shot_vs_sequential} shows a comparison between
the accuracy achieved with our single shot algorithm and the sequential
active learning on MNIST with LeNet5. The acquisitions batch size
of the sequential active learning were set to 100. Our algorithm ran
with $\lambda=t=0$. For the network training we used SGD with batch
size 128, leaning rate 0.1 and no l2 regularization. The SGD number
iterations is equivalent to $20$ epochs of the full train set.

Initially, our selection procedure shows a clear advantage. However,
once the training set grows large enough, the benefit of a sequential
setup starts to kick-in, the sequential algorithm starts to show superior
results. This experiment motivates further development of sequential
version of our algorithm.

\appendix

\section{\label{app:proofs}Missing Proofs}

\subsection{Proof of Proposition~\ref{prop:expected-risk}}
\begin{proof}
Let us write 
\[
\epsilon\coloneqq\left[\begin{array}{c}
\epsilon_{1}\\
\vdots\\
\epsilon_{n}
\end{array}\right]
\]
so
\[
\y=\matX\w+\epsilon\,.
\]
Thus, 
\[
\hat{\w}_{\lambda}=\matX_{\lambda}^{\pinv}\y=\matX_{\lambda}^{\pinv}\matX\w+\matX_{\lambda}^{\pinv}\epsilon=\matM_{\lambda}^{\pinv}\matM\w+\matX_{\lambda}^{\pinv}\epsilon
\]
and 
\[
\x^{\T}\w-\x^{\T}\hat{\w}_{\lambda}=\x^{\T}(\matI_{d}-\matM_{\lambda}^{\pinv}\matM)\w+\x^{\T}\matX_{\lambda}^{\pinv}\epsilon
\]
For brevity we denote $\matP_{\perp\matX}^{\lambda}=\matI_{d}-\matM_{\lambda}^{\pinv}\matM$.
Note that this is not really a projection, but rather (informally)
a ``soft projection''. So:
\begin{multline*}
(\x^{\T}\w-\x^{\T}\hat{\w}_{\lambda})^{2}=\w^{\T}\matP_{\perp\matX}^{\lambda}(\x\x^{\T})\matP_{\perp\matX}^{\lambda}\w\\
+\w^{\T}\matP_{\perp\matX}^{\lambda}(\x\x^{\T})\matX_{\lambda}^{\pinv}\epsilon+\epsilon^{\T}(\matX_{\lambda}^{\pinv})^{\T}(\x\x^{\T})\matX_{\lambda}^{\pinv}\epsilon
\end{multline*}
Finally,
\begin{flalign}
 & \Expect{R(\hat{\w}_{\lambda})}=\ExpectC{\x,\epsilon}{\left(\x^{\T}\w-\x^{\T}\hat{\w}_{\lambda}\right)^{2}}\label{eq:excess_risk}\\
 & =\ExpectC{\epsilon}{\ExpectC{\x}{\left(\x^{\T}\w-\x^{\T}\hat{\w}_{\lambda}\right)^{2}\,\,|\,\,\epsilon}}\nonumber \\
 & =\mathbb{E_{\epsilon}}\Big[\mathbb{E}_{\x}\Big[\w^{\T}\matP_{\perp\matX}^{\lambda}(\x\x^{\T})\matP_{\perp\matX}^{\lambda}\w+\ \w^{\T}\matP_{\perp\matX}^{\lambda}(\x\x^{\T})\matX_{\lambda}^{\pinv}\epsilon\nonumber \\
 & \qquad\qquad\qquad\qquad\qquad\qquad+\epsilon^{\T}(\matX_{\lambda}^{\pinv})^{\T}(\x\x^{\T})\matX_{\lambda}^{\pinv}\epsilon\,\,|\,\,\epsilon\Bigr]\Bigr]\nonumber \\
 & =\mathbb{E}_{\epsilon}\Big[\w^{\T}\matP_{\perp\matX}^{\lambda}\matC_{\rho}\matP_{\perp\matX}^{\lambda}\w+\w^{\T}\matP_{\perp\matX}^{\lambda}\matC_{\rho}\matX_{\lambda}^{\pinv}\epsilon\nonumber \\
 & \qquad\qquad\qquad\qquad\qquad\qquad\qquad\quad+\epsilon^{\T}(\matX_{\lambda}^{\pinv})^{\T}\matC_{\rho}\matX_{\lambda}^{\pinv}\epsilon\Big]\nonumber \\
 & =\w^{\T}\matP_{\perp\matX}^{\lambda}\matC_{\rho}\matP_{\perp\matX}^{\lambda}\w+\sigma^{2}\Trace{(\matX_{\lambda}^{\pinv})^{\T}\matC_{\rho}\matX_{\lambda}^{\pinv}}\nonumber \\
 & =\TNormS{\matC_{\rho}^{\nicehalf}\matP_{\perp\matX}^{\lambda}\w}+\sigma^{2}\Trace{\matC_{\rho}\matX_{\lambda}^{\pinv}(\matX_{\lambda}^{\pinv})^{\T}}\nonumber \\
 & =\TNormS{\mat C_{\rho}^{\nicehalf}\left(\mat I-\Mlambda^{+}\mat M\right)\w}+\sigma^{2}\Trace{\mat C_{\rho}\Mlambda^{+^{2}}\mat M}\nonumber \\
\nonumber 
\end{flalign}
\end{proof}

\subsection{Proof of Theorem~\ref{theorem:argmin_continues}}
\begin{proof}
Suppose $\bar{\w}\in\overline{\lim}_{\lambda\to\bar{\lambda}}\argmin_{\w}f\left(\w,\lambda\right)$.
The implies that there exits $\lambda_{n}\to\bar{\lambda}$ such that
$\w_{n}\in\argmin_{\w}f\left(\w,\lambda_{n}\right)$ and $\w_{n}\to\bar{\w}$.
From the continuity of $f$ we have that $f\left(\w_{n},\lambda_{n}\right)\to f\left(\bar{\w},\bar{\lambda}\right)$
. Now suppose for the sake of contradiction that $\bar{\w}\notin\argmin_{\w}f\left(\w,\bar{\lambda}\right)$
. So there is $\u$ such that $f(\u,\bar{\lambda})<f\left(\bar{\w},\bar{\lambda}\right)$.
From the continuity of $f$ in $\lambda$ there is $n_{0}$ such that
for all $n>n_{0}$ $f(\u,\lambda_{n})<f\left(\bar{\w},\bar{\lambda}\right)$.
Then from the continuity of $f$ in $\w$, and $\w_{n}\to\bar{\w}$,
for sufficiently large $n$, $f\left(\w_{n},\lambda_{n}\right)>f(\u,\lambda_{n})$,
which contradicts $\w_{n}\in\argmin_{\w}f\left(\w,\lambda_{n}\right)$.
\end{proof}

\section{\label{sec:alg-details}Details on the Algorithm}

We discuss the case of $\lambda=0$. The case of $\lambda>0$ requires
some more careful matrix algebra, so we omit the details.

Let us define 
\[
\matA_{j}\coloneqq\matK_{\calS^{(j)},\calS^{(j)}}^{-1},\quad\matB_{j}\coloneqq\matK_{:,\calS^{(j)}}^{\T}\matK_{:,\calS^{(j)}}
\]
and note that 
\begin{alignat*}{1}
J_{\lambda,t}(\calS^{(j)}) & =-\Trace{\matB_{j}(\matA_{j}-t\matA_{j}^{2})}.
\end{alignat*}
We also denote by $\tilde{\matA}_{j}$ and $\tilde{\matB}_{j}$ the
matrices obtained from $\matA_{j}$ and $\matB_{j}$ (respectively)
by adding a zero row and column.

Our goal is to efficiently compute $J_{\lambda,t}(\calS^{(j-1)}\cup\{i\})$
for any $i\in[m]-{\cal S}^{(j-1)}$ so find $i^{(j)}$ and form $\calS^{(j)}$.
We assume that at the start of iteration $j$ we already have in memory
$\matA_{j-1}$ and $\matB_{j-1}$. We show later how to efficiently
update $\matA_{j}$ and $\matB_{j}$ once we have found $i^{(j)}$.
For brevity, let us denote 
\begin{gather*}
\calS_{i}^{(j)}\coloneqq\calS^{(j-1)}\cup\{i\},\quad\matA_{ji}\coloneqq\matK_{\calS_{i}^{(j)},\calS_{i}^{(j)}}^{-1},\\
\matB_{ji}\coloneqq\matK_{:,\calS_{i}^{(j)}}^{\T}\matK_{:,\calS_{i}^{(j)}}
\end{gather*}
Let us also define
\begin{gather*}
\matC_{j-1}\coloneqq\tilde{\matB}_{j-1}\tilde{\matA}_{j-1},\quad\matD_{j-1}\coloneqq\tilde{\matB}_{j-1}\tilde{\matA}_{j-1}^{2},\\
\matE_{j-1}\coloneqq\tilde{\matA}_{j-1}^{2}
\end{gather*}
Again, we assume that at the start of iteration $j$ we already have
in memory $\matC_{j-1}$, $\matD_{j-1}$ and $\matE_{j-1}$, and show
how to efficiently update these matrices.

Let 
\[
\matW_{ji}\coloneqq\left[\begin{array}{cc}
0_{j-1} & \matK_{:,\calS^{(j-1)}}^{\T}\matK_{:,i}\\
\matK_{:,i}^{\T}\matK_{:,\calS^{(j-1)}} & \matK_{:,i}^{\T}\matK_{:,i}
\end{array}\right]
\]
and note that 
\[
\matB_{ji}=\tilde{\matB}_{j-1}+\matW_{ji}.
\]
Also important is the fact that $\matW_{ji}$ has rank 2 and that
finding the factors takes $O(mj)$ discounting the cost of computing
columns of $\matK$. Next, let us denote 
\[
r_{ji}=\frac{1}{(\matK_{ii}-\matK_{\calS^{(j)},i}^{\T}\matA_{j-1}\matK_{\calS^{(j)},i})}
\]
and 
\[
\matQ_{ji}\coloneqq r_{ji}\cdot\left[\begin{array}{cc}
\matA_{j-1}\matK_{\calS^{(j)},i}\matK_{\calS^{(j)},i}^{\T}\matA_{j-1}^{-1} & -\matA_{j-1}\matK_{\calS^{(j)},i}\\
-\matK_{\calS^{(j)},i}^{\T}\matA_{j-1} & 1
\end{array}\right]
\]
A well known identity regarding Schur complement implies that 
\[
\matA_{ji}=\tilde{\matA}_{j-1}+\matQ_{ji}
\]
Also important is the fact that $\matQ_{ji}$ has rank 2 and that
finding the factors takes $O(j^{2})$ discounting the cost of computing
entries of $\matK$.

So
\begin{flalign*}
 & J_{\lambda,t}(\calS_{i}^{(j)})=-\Trace{\matB_{ji}(\matA_{ji}-t\matA_{ji}^{2})}\\
 & =-\Trace{(\tilde{\matB}_{j-1}+\matW_{ji})(\tilde{\matA}_{j-1}+\matQ_{ji}-t(\tilde{\matA}_{j-1}+\matQ_{ji})^{2})}\\
 & =-\mathbf{Tr}\left((\tilde{\matB}_{j-1}+\matW_{ji})(\tilde{\matA}_{j-1}+\matQ_{ji}\right.)\\
 & \ \ \left.-t(\tilde{\matB}_{j-1}+\matW_{ji})(\tilde{\matA}_{j-1}^{2}+\matQ_{ji}^{2}+\tilde{\matA}_{j-1}\matQ_{ji}+\matQ_{ji}\tilde{\matA}_{j-1})\right)\\
 & =-\Trace{\matC_{j-1}+\tilde{\matB}_{j-1}\matQ_{ji}+\matW_{ji}(\tilde{\matA}_{j-1}+\matQ_{ji})}\\
 & \qquad+t\Trace{\matD_{j-1}+\tilde{\matB}_{j}(\tilde{\matA}_{j-1}\matQ_{ji}+\matQ_{ji}\tilde{\matA}_{j-1}+\matQ_{ji}^{2})}\\
 & \qquad\quad+\Trace{\matW_{i}(\matE_{j-1}+\matQ_{ji}^{2}+\tilde{\matA}_{j-1}\matQ_{ji}+\matQ_{ji}\tilde{\matA}_{j-1})}
\end{flalign*}
Now, $\matC_{j-1}$ is already in memory so $\Trace{\matC_{j-1}}$
can be computed in $O(j)$, $\matQ_{ji}$ has rank 2 and $\tilde{\matB}_{j-1}$
is in memory so $\Trace{\tilde{\matB}_{j-1}\matQ_{ji}}$ can be compute
in $O(j^{2}),$ and $\matW_{ji}$ has rank 2 and $\tilde{\matA}_{j-1}$
is in memory so $\Trace{\matW_{i}(\tilde{\matA}_{j-1}+\matQ_{ji})}$
can be computed in $O(j^{2})$. Using a similar rationale, all the
other terms of $J_{\lambda,t}(\calS_{i}^{(j)})$ can also be computed
in $O(j)$ or $O(j^{2})$, and overall $J_{\lambda,t}(\calS_{i}^{(j)})$
can be computed in $O(j^{2})$. Thus, scanning for $i^{(j)}$ takes
$O((m-j)j^{2})$.

Once $i^{(j)}$ has been identified, we set 
\[
{\cal S}^{(j)}=\calS_{i^{(j)}}^{(j)},\quad\matA_{j}=\matA_{ji^{(j)}}=\tilde{\matA}_{j-1}+\matQ_{ji^{(j)}}
\]
 and 
\[
\matB_{j}=\matB_{ji^{(j)}}=\tilde{\matB}_{j-1}+\matW_{ji^{(j)}}.
\]
The last two can be computed in $O(j^{2})$ once we form $\matQ_{i^{(j)}}$
and $\matW_{i^{(j)}}$. Computing the factors of these matrices takes
$O(mj)$. As for updating $\matC_{j-1}$, we have
\[
\matC_{j}=\tilde{\matC}_{j-1}+\tilde{\matB}_{j-1}\matQ_{ji^{(j)}}+\matW_{ji^{(j)}}\tilde{\matA}_{j-1}+\matW_{ji^{(j)}}\matQ_{ji^{(j)}}
\]
where $\tilde{\matC}_{j-1}$ is obtained from $\matC_{j-1}$ be adding
a zero row and column. Since $\matC_{j-1}$ is in memory and both
$\matQ_{ji^{(j)}}$ and $\matW_{i^{(j)}}$ have rank $O(1)$, we can
compute $\matC_{j}$ is $O(j^{2})$. Similar reasoning can be used
to show that $\matD_{j}$ and $\matE_{j}$ can also be computed in
$O(j^{2})$.

Overall, the cost of iteration $j$ is 
\[
O((m-j)(mj+j^{2}))=O(m^{2}j)
\]
 (since $j\leq m$). The cost of finding a design of size $n$ is
\[
O(m^{2}(n^{2}+D))
\]
 assuming the entire kernel matrix $\matK$ is formed at the start
and a single evaluation of $k$ takes $O(D)$.

\section{\label{sec:Experiment-setup-of}Experimental Setup for Results Reported
in Figure \ref{fig:opening}}

First, $\w\in\R^{100}$ was sampled randomly from $\Normal(0,\I)$
. Then a pool (the set from which we later choose the design) of 500
samples and a test set of 100 samples were randomly generated according
to $\x\sim\Normal(0,\Sigma)$, $\epsilon\sim\Normal(0,\sigma^{2}\I)$
and $y=\x^{\T}\w+\epsilon$, where $\Sigma\in\R^{100\times100}$ is
diagonal with $\Sigma_{ii}=\exp(-\nicefrac{2.5i}{100})$, and $\sigma=0.2$.
We then created three incremental designs (training sets) of size
120 according to three different methods:
\begin{itemize}
\item Random design - at each iteration we randomly choose the next training
sample from the remaining pool.
\item Classical OED (variance oriented) - at each iteration we choose the
next training sample from the remaining pool with a greedy step that
minimizes the variance term in Eq. \eqref{eq:psibar_lambda-1}.
\item Overparameterized OED - at each iteration we chose the next training
sample from the remaining pool with a greedy step that minimizes Eq.
\eqref{eq:psibar_lambda-1}, with $\lambda=0$ and $t=\sigma^{2}$
.
\end{itemize}
With the addition of each new training sample we computed the new
MSE achieved on the test set with minimum norm linear regression.

\bibliographystyle{ieeetr}
\bibliography{bibtex}

\end{document}